\documentclass[10pt]{article}
\usepackage[accepted]{tmlr}
\usepackage{wrapfig}

\usepackage{amsmath,amsfonts,bm}

\def\eqref#1{equation~\ref{#1}}

\def\1{\bm{1}}

\DeclareMathAlphabet{\mathsfit}{\encodingdefault}{\sfdefault}{m}{sl}
\SetMathAlphabet{\mathsfit}{bold}{\encodingdefault}{\sfdefault}{bx}{n}

\newcommand{\E}{\mathbb{E}}

\newcommand{\R}{\mathbb{R}}

\usepackage[hidelinks]{hyperref}
\usepackage{url}
\usepackage{amssymb}
\usepackage{mathtools}
\usepackage{mathrsfs}
\usepackage{xcolor}
\usepackage{graphicx}
\usepackage{subcaption}
\usepackage[space]{grffile}
\usepackage{url}
\usepackage[linesnumbered,algoruled,longend,vlined]{algorithm2e}
\usepackage{booktabs}
\usepackage{multirow}
\usepackage{amsthm}
\usepackage{placeins}

\newtheorem{lemma}{Lemma}
\newtheorem{proposition}{Proposition}

\renewcommand{\eqref}[1]{\textup{(\ref{#1})}}

\usepackage{microtype}
\newcommand{\Dataset}{\mathcal{D}}
\newcommand{\Simplex}{\Delta^{K-1}}
\newcommand{\TrueCond}{q(\cdot \mid x)}
\newcommand{\ValuePi}{V^\pi}
\newcommand{\Critic}{V_\phi}
\newcommand{\Geom}{\mathrm{Geom}}
\newcommand{\Breg}{D_F}

\title{Do Not Imitate, Reinforce:\\ Iterative Classification via Belief Refinement}

\author{%
    Mahdi Kallel \email{} \\ 
    \addr{Center for Artificial Intelligence and Data Science \\ 
    University of Würzburg}
    \AND
    Johannes Tölle\thanks{This work was carried out as part of Johannes Tölle's master's thesis at the University of Würzburg.} \email{} \\ 
    \addr{Automation and Robotics Research Group, SnT \\
    University of Luxembourg}
    \AND
    Ahmed Hendawy \email{} \\
    \addr{Department of Computer Science \\ 
    Technical University of Darmstadt}
    \AND 
    Carlo D'Eramo \email{} \\ 
    \addr{Center for Artificial Intelligence and Data Science \\ 
    University of Würzburg}
}

\def\month{10}
\def\year{2026}
\def\openreview{\url{https://openreview.net/forum?id=7xXijNfm8j}}

\begin{document}

\maketitle
\vspace{-16pt}

\begin{abstract}

Standard supervised classification trains models to imitate the labels in their training data. Since this typically happens in a single forward pass, models are locked into a fixed compute budget whether an input is simple or complex. Moreover, the imitation objective pushes the model toward full certainty on its training data, which often carries over into evaluation, leading to overconfident predictions. Rather than forcing a rigid, single-step guess, early learning systems often framed learning as a trial-and-error process of gradual improvement. Reclaiming this perspective for classification, we propose Reinforced Iterative Classification (RIC), a Reinforcement Learning (RL) framework in which a recurrent actor iteratively refines a predictive distribution and a critic estimates the value of continued refinement. At inference time, the learned value function provides a principled stopping rule, halting when further refinement is no longer useful. We study the theoretical properties of this formulation and evaluate an implementation trained with Simple Policy Optimization (RIC-SPO). Across ten image classification benchmarks, including medical images, our approach exhibits fundamentally different learning and generalization dynamics from its supervised counterparts and achieves a better balance between accuracy and calibration.

\end{abstract}

\section{Introduction}
\label{sec:intro}
Modern classification models are trained to predict the class distribution in a single forward pass.
 From a sequential decision-making point of view, this represents a rigid form of imitation learning in which the model allocates fixed computation to every input to replicate the observed training labels in one shot. 
 Although effective, this training scheme couples label fitting with confidence growth. Easy inputs, ambiguous inputs, and corrupted labels are all pushed toward the observed label, so high training accuracy is usually accompanied by increasingly confident predictions \citep{guo2017calibration,wei2022mitigating}. 
This raises a natural question: can classification instead be learned as a sequence of prediction refinements, with computation allocated according
to how much an input still stands to gain?

Reinforcement Learning (RL) is well suited to this question through its emphasis on learning from evaluative feedback \citep{sutton2018reinforcement}, a perspective rooted in early trial-and-error learning systems \citep{samuel1959some,barto2019reinforcement}. Applications to classification \citep{wiering2011reinforcement} have used RL to learn where to look \citep{mnih2014recurrent}, which layers to execute \citep{wang2018skipnet}, which features to acquire \citep{dulacarnold2011datum}, and which discrete label to assign \citep{lin2020deep}. These methods use RL to control computation or make discrete classification decisions, while the full predictive distribution itself remains outside the action space. Treating this distribution as the action being refined places RL at the core of learning probabilistic predictions and raises the question of how sequential reinforcement alters fitting, confidence, and generalization.

We address this question with Reinforced Iterative Classification (RIC), an RL framework that treats classification as sequential prediction refinement (Figure~\ref{fig:nets}). At each step, a recurrent actor proposes a class probability vector and is rewarded for the increase in log probability assigned to the observed label. Although this objective recovers the cross-entropy prediction target in the deterministic limit, the policy learns through returns earned by stochastic refinements rather than direct differentiation of each prediction's loss. A learned critic estimates the improvement expected from further refinement. At inference, refinement stops once this estimate reaches zero, yielding adaptive computation without a ponder penalty or a separate halting network.

\begin{figure}[tb]
    \captionsetup{skip=4pt}
    \centering
    \includegraphics[width=0.8\linewidth]{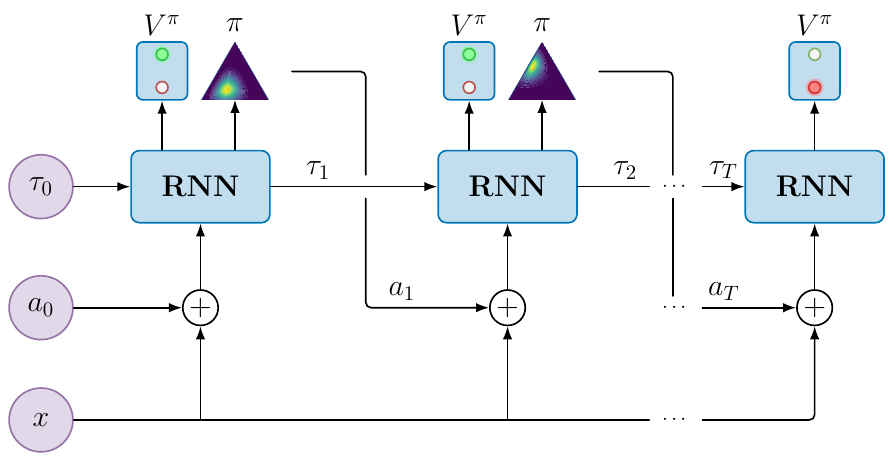}
    \caption{\textbf{Recurrent actor-critic architecture for RIC.} The recurrent module updates the latent state from the input and previous prediction. The policy head samples the next class distribution, and the value head predicts future improvement and supplies the stopping rule at inference.}
    \label{fig:nets}
    \vspace{-6pt}
\end{figure}

To examine how this learning process affects fitting and generalization, we introduce RIC-SPO, an implementation trained using Simple Policy Optimization (SPO) \citep{xie2025simple}. Our training analyses reveal a different learning regime in which RIC-SPO fits the training data more gradually, yet attains higher validation accuracy than supervised learning at matched training accuracy. Predictive accuracy rises ahead of confidence, while calibration improves during training and remains stable at later stages. Across ten image classification benchmarks, including noisy human annotations and medical images, RIC-SPO achieves a better trade-off between accuracy and calibration than single-pass and adaptive supervised baselines.

We make three contributions. \textbf{(i)} We formulate classification as a Markov Decision Process (MDP) in which full class probability distributions are actions and increases in label log probability define rewards. \textbf{(ii)} We show that discounting trains useful intermediate predictions, characterize the optimal prediction in the deterministic limit, and establish a stopping criterion based on the value of continued refinement. \textbf{(iii)} We introduce RIC-SPO and evaluate its training dynamics, predictive performance, and adaptive computation across ten vision benchmarks.\footnote{Code is available at \url{https://github.com/kallel-mahdi/ric}.}

\section{Preliminaries}
\label{sec:pre}
\noindent\textbf{Supervised classification.} We consider a $K$-class classification problem with input space $\mathcal{X}$, label space $\mathcal{Y}=\{1,\dots,K\}$, and an unknown data distribution $\Dataset$ over $\mathcal{X}\times\mathcal{Y}$.
We write $\Simplex=\{p\in\R^K_{\ge 0}\mid\sum_{j=1}^K p_j=1\}$ for the probability simplex and $\TrueCond \in \Simplex$ for the true conditional distribution over labels given input $x$.
Cross-entropy is $\mathrm{CE}(p,y)=-\log p_y$, where $p_y$ is the probability that $p$ assigns to class $y$. We use NLL for its empirical mean. For labels $y\sim q$, its expectation decomposes as $\E_{y\sim q}[\mathrm{CE}(p,y)]=H(q)+\mathrm{KL}(q\,\|\,p)$, with Shannon entropy $H(q)$ and Kullback--Leibler divergence $\mathrm{KL}(q\,\|\,p)$.
We measure calibration with expected calibration error (ECE)~\citep{guo2017calibration}. A sample of $n$ predictions is grouped into $M$ equal-width confidence bins $B_1,\dots,B_M$, and $\mathrm{ECE}=\sum_{m=1}^{M}\frac{|B_m|}{n}\,|\mathrm{acc}(B_m)-\mathrm{conf}(B_m)|$, the weighted average gap between bin accuracy and mean confidence.

\smallskip
\noindent\textbf{Markov decision processes.} A standard Markov decision process (MDP) is a tuple $\mathcal{M}=\langle \mathcal{S},\mathcal{A},P,R,\gamma \rangle$ with state space $\mathcal{S}$, action space $\mathcal{A}$, transition kernel $P$, reward function $R$, and discount factor $\gamma\in(0,1)$.
An agent's behavior is governed by a policy $\pi$ that maps states to distributions over actions.
The value function 
$\ValuePi(s_t)=\E_\pi[\sum_{k=1}^{\infty}\gamma^{k-1}r_{t+k}\mid s_t]$
represents the expected cumulative discounted reward from state $s_t$ onward.
Actor-critic methods learn a parameterized policy $\pi_\theta$ (the actor) together with a critic $\Critic$ that approximates $\ValuePi$ and serves as a baseline for the policy update.

\section{Reinforced Iterative Classification}
\label{sec:ric}

RIC turns classification into an iterative decision process in which predictions are refined over time. The framework combines a classification MDP, a recurrent actor-critic, and a value-based halting rule. Starting from a uniform distribution, the actor proposes a new class distribution at each step and is rewarded for improving the probability assigned to the observed label. The critic estimates the expected improvement remaining and provides a principled signal for adaptive inference. These components define RIC independently of the learning algorithm used to optimize them. We train RIC with SPO \citep{xie2025simple} and call the resulting method RIC-SPO.

\subsection{Classification MDP}
\label{sec:mdp}
RIC turns each labeled example into a small decision problem inside the classifier, where the agent acts on its own class distribution. A training episode is a single example
$(x, y) \sim \Dataset$. The reward measures the increase in log probability assigned to label $y$.

\paragraph{Action and state.} Each episode begins from $s_0=(x,a_0,\tau_0)$, where $a_0=[1/K,\dots,1/K]$ is the uniform prior over classes and $\tau_0=\mathbf{0}$ is the initial latent state. At refinement step $t\ge 0$, the state $s_t=(x,a_t,\tau_t)$ contains the input, the current prediction, and the current latent state. The policy samples the next prediction $a_{t+1}\sim\pi_\theta(\cdot\mid s_t)$ with $a_{t+1}\in\Simplex$. The label is used only to compute reward, so the policy conditions on $s_t$.

\paragraph{Transition.} The dynamics are deterministic apart from the policy sample. Given $s_t$, the recurrent module computes $\tau_{t+1}=f_\theta(x,\tau_t,a_t)$. The sampled prediction then yields the next state $s_{t+1}=(x,a_{t+1},\tau_{t+1})$. The latent state lets the classifier remember how its belief has been refined so far.

\paragraph{Reward.}
Standard classification gives the model a loss on a completed prediction.
RIC pays each refinement step for the improvement it makes.
The dense signal eases credit assignment across refinement steps.
Let $a_{t,y}$ denote the probability that prediction $a_t$ assigns to label $y$.
For $t \ge 1$, we define
\begin{equation}
r_t = \log a_{t,y} - \log a_{t-1,y}. \label{eq:reward}
\end{equation}
The label is fixed within the episode, so each pair $(x,y)$ induces its own MDP
in which $y$ is part of the environment.
Let
\begin{equation}
J_{x,y}(\theta)
=
\mathbb{E}_{\pi_\theta}
\left[
\sum_{t=1}^{\infty} \gamma^{t-1} r_t
\,\middle|\, x,y
\right]
\label{eq:return}
\end{equation}
be the discounted return on a single example. RIC maximizes the average $J(\theta)=\E_{(x,y)\sim \Dataset}\left[J_{x,y}(\theta)\right]$.

\subsection{Architecture and RIC-SPO}
\label{sec:architecture}

\noindent\textbf{Architecture.} The architecture mirrors the MDP. A recurrent backbone carries the current latent state $\tau_t$ and produces $\tau_{t+1}$ with $x$ and $a_t$ as input (Figure~\ref{fig:nets}).
Then, a policy head samples the next class distribution, and a value head predicts how much refinement remains useful.
The policy head parameterizes a Dirichlet distribution $\pi_\theta(\cdot \mid s_t)$ over $\Simplex$ by producing a mean vector $\mu\in\Simplex$ and a concentration $c$. During training, it samples the next prediction as $a_{t+1}\sim\mathrm{Dir}(c\mu)$.\footnote{With the numerical offset $\epsilon=0.01$, deterministic evaluation uses the realized mean $(c\mu+\epsilon)/(c+K\epsilon)$; the analysis sets $\epsilon=0$.}
The value head outputs $\Critic(s_t)\in\R$. Both heads are reused at every refinement step and share the recurrent backbone.

\looseness=-1 \noindent\textbf{Inference.} The reward of Equation~\eqref{eq:reward} gives the value function a concrete meaning. It estimates the cumulative improvement still available from a state. At inference, RIC evaluates $\Critic(s_t)$ before spending another refinement step. A positive value continues refinement. A nonpositive value stops and returns the current prediction. We evaluate both stochastic rollouts and deterministic rollouts that use the Dirichlet mean.

\looseness=-1 \noindent\textbf{Training.} RIC can be trained using different policy optimization algorithms. In this paper, we use SPO as the policy update rule and call the resulting method RIC-SPO.
During training, each episode is unrolled for a fixed horizon of $T$ steps, and value-based stopping is reserved for inference. 
Training proceeds in rounds. We freeze a snapshot $(\theta_{\mathrm{old}},\phi_{\mathrm{old}})$, collect rollouts with $\pi_{\theta_{\mathrm{old}}}$ for $P$ passes, and update the actor and critic with SPO \citep{xie2025simple} and generalized advantage estimation \citep{schulman2016high}. SPO is a PPO-style update \citep{schulman2017proximal} that bounds how far the policy moves from the snapshot within each round. Algorithm~\ref{alg:recurrent-spo} summarizes this loop.

\begin{algorithm}[tb]
  \caption{RIC-SPO}
  \label{alg:recurrent-spo}
  \KwIn{dataset $\Dataset$, horizon $T$, passes per snapshot $P$,
        discount $\gamma$, value coefficient $c_v$}

  Initialize parameters $\theta$, $\phi$ \;
  \For{$n \gets 1, 2, \ldots$}{
    $\theta_{\mathrm{old}} \gets \theta$, \enspace $\phi_{\mathrm{old}} \gets \phi$ \tcp*{snapshot}
    \For{$p \gets 1$ \KwTo $P$}{
      \ForEach{minibatch $\mathcal{B} \subset \Dataset$}{
        Roll out $\pi_{\theta_{\mathrm{old}}}$ for $T$ steps on each $x \in \mathcal{B}$\;
        $\hat{A}, \hat{R} \gets \mathrm{GAE}(\mathcal{B}, V_{\phi_{\mathrm{old}}}, \gamma)$ \tcp*{advantage and return calculation}
        Update $\theta, \phi$ via $\nabla \hat{\mathbb{E}}_{\mathcal{B}}\!\bigl[L_{\mathrm{SPO}}(\theta, \theta_{\mathrm{old}}, \hat{A}) - c_v \lVert V_\phi - \hat{R} \rVert^2\bigr]$ \tcp*{update model parameters}
      }
    }
  }
\end{algorithm}

\section{Theoretical Analysis}
\label{sec:analysis}
RIC rewrites classification as an RL problem, so the return must have a statistical interpretation. The analysis gives that interpretation in three linked steps. Discounting evaluates predictions at a random refinement depth, the deterministic limit selects the true conditional distribution, and the critic compares current prediction error with expected future prediction error.
The propositions characterize the RIC objective and value function independently of the learning algorithm used to train the policy.

\subsection{Discounting Trains Predictions at Every Depth}
\label{sec:landscape}

\looseness=-1 Discounting determines which refinement depths are trained. Because reward is the change in label log probability, discounted return evaluates the prediction reached at a random stopping step. Earlier refinements receive more weight, but every depth has positive probability.\enlargethispage{\baselineskip}

\begin{lemma}[Geometric horizon \citep{puterman1994markov}]
\label{lem:geom_discount}
For any integrable reward sequence $\{r_t\}_{t\ge 1}$ and discount $\gamma \in (0,1)$,
\[
\E\!\left[\sum_{t=1}^{\infty}\gamma^{t-1} r_t\right] = \E_{N\sim\Geom(1-\gamma)}\!\left[\sum_{t=1}^{N}r_t\right],
\]
where the random horizon $N$ is independent of the rewards.
\end{lemma}

For RIC, the reward is a difference of label log probabilities, so the cumulative reward telescopes. Applying Lemma~\ref{lem:geom_discount} to the single example return of Equation~\eqref{eq:return} gives
\begin{equation}
\begin{aligned}
J_{x,y}(\theta)
&=
\E_{N\sim\Geom(1-\gamma),\,\pi_\theta}\!\left[
\sum_{t=1}^{N}
\bigl(\log a_{t,y}-\log a_{t-1,y}\bigr)
\,\middle|\,x,y
\right] \\
&=
\E_{N\sim\Geom(1-\gamma),\,\pi_\theta}\!\left[
\log a_{N,y}-\log a_{0,y}
\,\middle|\,x,y
\right] \\
&=
\E_{N\sim\Geom(1-\gamma),\,\pi_\theta}\!\left[
\log a_{N,y}\mid x,y
\right]
-\log a_{0,y}.
\end{aligned}
\label{eq:telescope}
\end{equation}
\looseness=-1 Equation~\eqref{eq:telescope} evaluates the prediction reached at a random refinement depth. The return is high when that prediction assigns higher log probability to the label. The term $-\log a_{0,y}$ is constant. It equals $\log K$ under the uniform prior $a_0=[1/K,\dots,1/K]$ and leaves the maximizing policy unchanged.
Thus the discounted RL objective trains prediction quality at random refinement depths. This makes RIC an anytime classifier, with supervision supplied by reward in a form close to deep supervision \citep{lee2015deeply}.
\subsection{The Optimal Prediction}
\label{sec:target}

\looseness=-1 Once discounting has placed supervision across refinement depths, the remaining issue is which prediction the expected objective favors. For a deterministic policy, it is the true conditional distribution selected by cross-entropy. For a stochastic policy, expected cross-entropy also includes a cost for variation around the mean prediction.

\begin{lemma}[Cross-entropy decomposition for stochastic predictions \citep{gneiting2007strictly}]
\label{lem:jensen_gap}
Fix a state $s$ with true conditional distribution $q=q(\cdot\mid x)$. Let $a\sim\pi(\cdot\mid s)$ have mean $\mu=\E_\pi[a\mid s]$ and assume $a_y>0$ on the policy support. Then
\begin{equation}
\label{eq:decomp}
\E_{y\sim q,\,a\sim\pi}\!\left[-\log a_y \mid s\right]
= H(q) + \mathrm{KL}(q\,\|\,\mu)
+ \E_{y\sim q,\,a\sim\pi}\!\left[\Breg(a_y\,\|\,\mu_y)\mid s\right],
\end{equation}
where $\Breg(x\,\|\,y)=x/y-\log(x/y)-1\ge 0$. The three terms are label ambiguity, mean prediction error, and the cost of stochastic actions.
\end{lemma}

\looseness=-1 Equation~\eqref{eq:decomp} separates fixed label ambiguity from two reducible terms: error in the mean prediction and the cost of stochastic actions. Both vanish when the mean prediction is $\TrueCond$ and the policy becomes deterministic. Since Equation~\eqref{eq:telescope} assigns positive weight to every depth, this target holds throughout refinement.

\begin{proposition}[Optimal prediction target]
\label{prop:target}
If the policy class can approach mean prediction $\TrueCond$ while the final term in Equation~\eqref{eq:decomp} vanishes uniformly over refinement steps, the expected RIC objective approaches its supremum. Along any sequence approaching the supremum, the mean prediction converges to $\TrueCond$ and this term vanishes at every fixed refinement step.
\end{proposition}

\paragraph{Proof sketch.}
Equation~\eqref{eq:decomp} identifies two nonnegative reducible terms, and Equation~\eqref{eq:telescope} gives every depth positive weight. Approaching the upper bound therefore forces both terms to vanish at every fixed depth. The full proof is in Appendix~\ref{sec:target_proof}.

Thus deterministic RIC and cross-entropy select the same probability target.
Cross-entropy pushes the predictions toward its label for as long as any gap remains.
In contrast, RIC trains the same prediction through the policy gradient, which moves it only while stochastic refinements earn distinguishable returns.
The two objectives therefore differ in the path taken to that target rather than in the target itself.

\subsection{The Critic Measures Remaining Prediction Improvement}
\label{sec:value_loss}

\looseness=-1 At inference, the learned critic determines whether refinement should continue. Proposition~\ref{prop:value} gives this decision a statistical meaning. The critic compares the prediction available now with the prediction expected after more refinement.
Stopping at state $s_t$ keeps the current prediction $a_t$. Its expected cross-entropy is $H(q(\cdot\mid x))+L_t$, where $L_t=\mathrm{KL}(\TrueCond\|a_t)$ is the current KL error. Continuing under $\pi$ replaces that error by the expected error at the random future depth. The label ambiguity $H(q(\cdot\mid x))$ is tied to the input, so it is identical in both losses and cancels. The critic therefore estimates whether refinement lowers the KL error.

\begin{proposition}[Value of refinement]
\label{prop:value}
For any policy $\pi$ and state $s_t=(x,a_t,\tau_t)$, let $L_j=\mathrm{KL}(\TrueCond\|a_j)$ denote the KL error at refinement depth $j$. The value function under the reward of Equation~\eqref{eq:reward} satisfies
\begin{equation}\label{eq:valueid}
\ValuePi(s_t)\;=\;L_t\;-\;\E_{N\sim\Geom(1-\gamma),\,\pi}\!\left[\,L_{t+N}\mid s_t\,\right].
\end{equation}
\end{proposition}

\paragraph{Proof sketch.}
Apply Lemma~\ref{lem:geom_discount} from $s_t$ and average the unknown label under $q(\cdot\mid x)$. The cross-entropy decomposition turns the current and future predictions into entropy plus KL error. The common entropy cancels and leaves Equation~\eqref{eq:valueid}. The full proof is in Appendix~\ref{sec:value_proof}.

Equation~\eqref{eq:valueid} is the stopping rule used in Section~\ref{sec:architecture}. If $L_t$ is larger than the expected future error, then $\ValuePi(s_t)>0$ and refinement is expected to help. If $L_t$ is not larger, then $\ValuePi(s_t)\le 0$ and the current prediction is at least as good as continuing.
The cancellation of the input-dependent label ambiguity $H(q(\cdot \mid x))$ explains why state value is distinct from predictive confidence. Because the critic estimates the expected error reduction under the current policy, it separates label ambiguity from model error. For highly ambiguous inputs, the policy can halt early with a high-entropy prediction as soon as all reducible $\text{KL}$ error is eliminated. 
Conversely, a confident but incorrect prediction will maintain a positive value as long as subsequent refinement steps can steer the distribution closer to $q$. This mechanism instantiates the classical monitoring rule for anytime algorithms \citep{russell1991principles,hansen2001monitoring}, and Proposition~\ref{prop:value} shows that the critic trained alongside the policy already estimates it.

\section{Results}
\label{sec:res}
\begin{figure}[b]
    \centering
    \includegraphics[width=\textwidth]{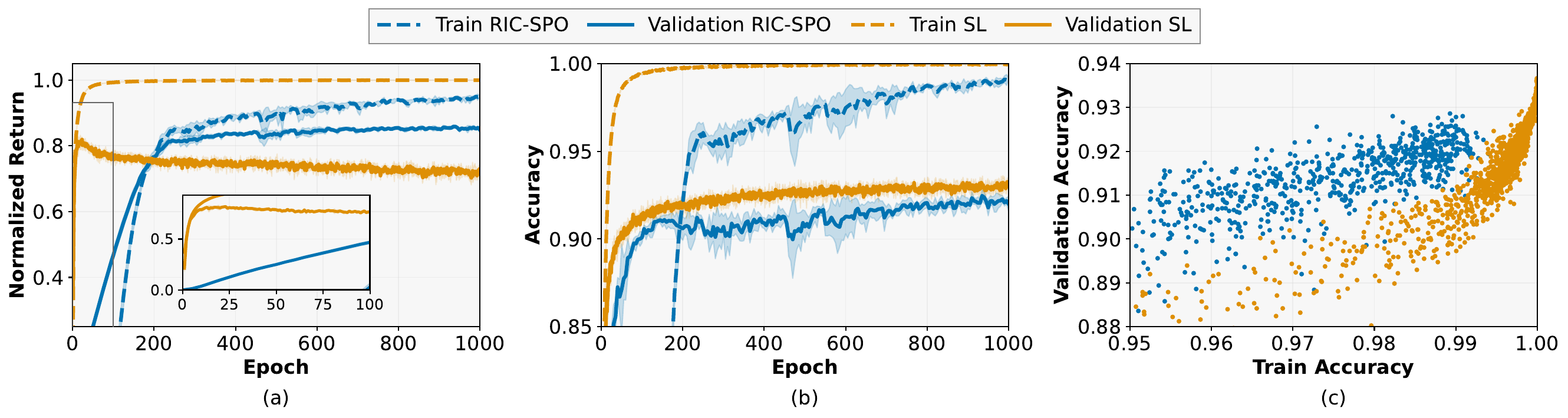}
    \caption{
\textbf{CIFAR-10 training dynamics of RIC-SPO and SL.}
Trained for $1000$ epochs with $5$ random seeds.
(a) Normalized return over training: discounted RIC reward divided by $\log K$, with the one-step analogue for SL. RIC-SPO improves more slowly because policy gradient updates have higher variance, but attains higher normalized validation return.
(b) Classification accuracy on train and validation sets. Both objectives optimize predictive performance, and RIC-SPO reaches comparable accuracy.
(c) Evaluation accuracy as a function of training accuracy across epochs. At matched training accuracy, RIC-SPO reaches higher evaluation accuracy.
}
    \label{fig:dynamics}
\end{figure}

The analysis characterizes RIC's prediction target and value-based stopping rule.
We empirically evaluate RIC-SPO under deterministic and stochastic evaluation against single-pass supervised learning (SL), ACT \citep{graves2016adaptive}, PonderNet \citep{banino2021pondernet}, and Evidential Deep Learning (EDL) \citep{sensoy2018evidential} and assess it as a practical image classifier.
Our results are organized to \textbf{(i)} examine training dynamics and confidence timing, \textbf{(ii)} benchmark test set accuracy and calibration across datasets, and \textbf{(iii)} analyze value-based halting.

\subsection{Setup}
\label{sec:setup}

 All methods use the same ResNet-34 backbone \citep{he2016deep} to encode an image into an embedding and are trained end-to-end.
  The baselines are chosen so that each shares one ingredient with RIC-SPO.
  SL predicts from the shared embedding in a single forward pass and uses the standard cross-entropy objective.
  ACT and PonderNet use the same GRU refinement module under supervised adaptive computation objectives.
  PonderNet's halting prior is matched to the RIC discount ($\lambda_p = 1-\gamma$).
  EDL shares the Dirichlet output family.
  Full hyperparameter settings and sweeps are given in the appendix.
  All checkpoints are selected by the lowest validation loss. In addition, Appendix~\ref{app:selector-sensitivity} repeats the SL, EDL, and RIC-SPO comparison with checkpoints selected by validation accuracy.
  We treat deterministic evaluation as the default deployment setting, since a classifier ordinarily returns one stable prediction for each input. Unless stated otherwise, our plots use this setting; we report stochastic evaluation separately to show the effect of sampling at test time.
  RIC-SPO uses a Dirichlet policy over the simplex with discount $\gamma{=}0.8$ and truncated rollouts of $T{=}20$ refinement steps with value bootstrapping.
  At $\gamma{=}0.8$, only about $1\%$ of the discount weight remains beyond the $20$-step rollout horizon.

  We evaluate on CIFAR-10 \citep{krizhevsky2009learning}, CIFAR-10N \citep{wei2021learning}, Galaxy10 \citep{leung2024galaxy10}, and ImageWoof
  \citep{howard2020imagenette}, a fine-grained subset of ImageNet \citep{deng2009imagenet}.
  We also include five MedMNIST subsets \citep{yang2023medmnist}, none of which any method solves beyond $90\%$ accuracy.
  In all plots, shaded bands show 95\% confidence intervals across random seeds.

\begin{figure}[b]
    \centering
    \includegraphics[width=\textwidth]{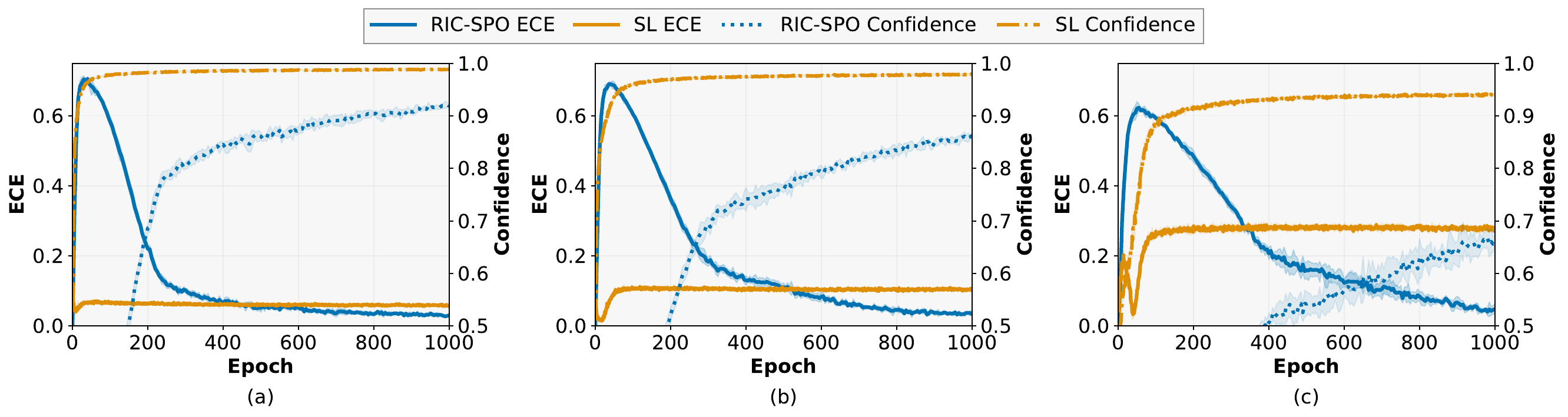}
    \caption{
\textbf{Calibration dynamics under varying label noise (deterministic evaluation).}
Validation ECE and confidence over training time on CIFAR-10 (a), CIFAR-10N with aggregated human labels (9.03\% noise) (b), and CIFAR-10N with worst-case labels (40.21\% noise) (c).
}
    \label{fig:noise}
\end{figure}

\subsection{\texorpdfstring{Training Dynamics and Confidence Timing}{Training Dynamics and Confidence Timing}}
\label{sec:dynamics}

\begin{figure}[tb]
    \centering
    \includegraphics[width=\textwidth]{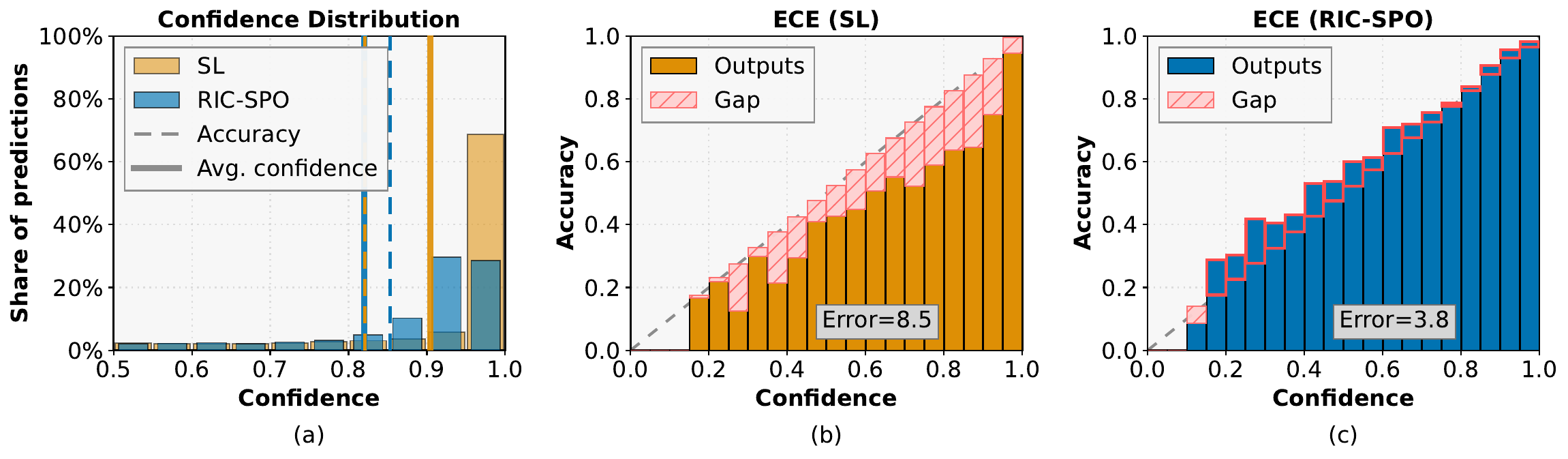}
    \caption{\textbf{Confidence distributions and reliability diagrams on ImageWoof (test).}
    The confidence histogram and reliability diagrams aggregate five models selected by validation loss; RIC-SPO is evaluated deterministically.
    Deterministic RIC-SPO raises accuracy from $82.1\%$ to $85.3\%$ and reduces ECE from $0.085$ to $0.038$ relative to SL.}
\label{fig:calibration}
\end{figure}

\looseness=-1 RIC-SPO attains higher normalized validation return and maintains it through the end of training, while SL's normalized validation return declines (Figure~\ref{fig:dynamics}(a)).
Its policy-gradient updates have higher variance and reach the high-accuracy regime more gradually (Figure~\ref{fig:dynamics}(b)).
At matched training accuracy, RIC-SPO reaches higher evaluation accuracy than SL (Figure~\ref{fig:dynamics}(c)).
RIC-SPO also maintains low validation loss and stable calibration later in training, while SL reaches its best validation loss and ECE early and degrades afterward.

Figure~\ref{fig:noise} tracks deterministic evaluation over training as label noise increases. The SPO policy update caps how far the class distribution can move between snapshots. Consistent with this, confidence evolves more slowly than accuracy.
RIC-SPO is therefore underconfident early, while supervised and evidential objectives let confidence rise as fast as fitting allows.
ECE peaks above $0.6$ early in training, before confidence catches up to the accuracy already reached.
SL confidence saturates within tens of epochs, and its ECE grows with label noise. Validation-loss selection recovers its brief calibrated window in Table~\ref{tab:non-medmnist-acc-ece-by-method}.
On standard CIFAR-10, RIC-SPO reaches a lower minimum ECE than SL, and this minimum occurs at a later stage when validation accuracy is already high (Figures~\ref{fig:noise}(a) and~\ref{fig:dynamics}(b)), so calibration improves alongside predictive performance during this phase.
RIC-SPO confidence rises toward a lower asymptote as noise increases, which keeps predictions on noisy examples away from the simplex boundary, and deterministic evaluation ends training with the lower ECE in all three settings.

\looseness=-1 This timing difference is the central training signature of RIC-SPO. Predictive accuracy improves faster than model confidence. Cross-entropy couples the two and reaches peak accuracy when confidence is already extreme. As Proposition~\ref{prop:target} predicts, both objectives fit the same target and reach comparable accuracy (Figure~\ref{fig:dynamics}(b)), but they take different paths to it. The separation is what leaves room for calibration. On clean data, RIC-SPO reaches high accuracy while confidence is still rising, so longer training continues to increase confidence. Under label noise, the target itself is less concentrated, so confidence rises toward a lower level and mislabeled examples keep moderate probabilities.

\subsection{\texorpdfstring{Test Set Accuracy and Calibration}{Test Set Accuracy and Calibration}}
\label{sec:evaluation}

\begin{table}[tb]
  \centering
  \footnotesize
  \setlength{\tabcolsep}{3.0pt}
  \renewcommand{\arraystretch}{1.08}
  \begin{tabular}{@{}ll*{5}{c}@{}}
    \toprule
    Method & Metric & CIFAR-10 & \shortstack{CIFAR-10N\\(Agg.)} & \shortstack{CIFAR-10N\\(Worst)} & ImageWoof & Galaxy10 \\
    \midrule
    \textbf{SL} & Acc. [\%] & $88.63{\pm}0.46$ & $85.68{\pm}0.49$ & $\mathbf{75.44{\pm}1.70}$ & $82.05{\pm}0.87$ & $80.43{\pm}0.67$ \\
     & ECE & $0.049{\pm}0.004$ & $\mathbf{0.011{\pm}0.002}$ & $0.176{\pm}0.012$ & $0.085{\pm}0.009$ & $0.063{\pm}0.012$ \\
    \addlinespace[2pt]
    \textbf{EDL} & Acc. [\%] & $91.85{\pm}0.36$ & $86.17{\pm}0.50$ & $59.30{\pm}10.26$ & $84.17{\pm}0.42$ & $81.36{\pm}0.54$ \\
     & ECE & $0.047{\pm}0.002$ & $0.112{\pm}0.012$ & $0.291{\pm}0.103$ & $0.072{\pm}0.004$ & $0.100{\pm}0.009$ \\
    \addlinespace[2pt]
    \textbf{RIC-det.} & Acc. [\%] & $\mathbf{92.57{\pm}0.19}$ & $\mathbf{88.04{\pm}0.17}$ & $73.61{\pm}0.94$ & $\mathbf{85.34{\pm}0.35}$ & $\mathbf{83.09{\pm}0.47}$ \\
     & ECE & $\mathbf{0.030{\pm}0.004}$ & $0.081{\pm}0.022$ & $0.284{\pm}0.019$ & $0.038{\pm}0.004$ & $0.068{\pm}0.021$ \\
    \addlinespace[2pt]
    \textbf{RIC-stoch.} & Acc. [\%] & $89.04{\pm}0.33$ & $85.43{\pm}0.36$ & $64.13{\pm}1.36$ & $82.82{\pm}0.81$ & $80.43{\pm}0.51$ \\
     & ECE & $\mathbf{0.030{\pm}0.007}$ & $0.050{\pm}0.009$ & $\mathbf{0.051{\pm}0.006}$ & $\mathbf{0.017{\pm}0.003}$ & $\mathbf{0.037{\pm}0.009}$ \\
    \addlinespace[2pt]
    \textbf{ACT} & Acc. [\%] & $88.85{\pm}0.41$ & $86.37{\pm}0.36$ & $\mathbf{76.86{\pm}0.86}$ & $81.71{\pm}0.44$ & $80.28{\pm}1.33$ \\
     & ECE & $0.050{\pm}0.005$ & $0.016{\pm}0.002$ & $0.196{\pm}0.005$ & $0.079{\pm}0.012$ & $\mathbf{0.054{\pm}0.010}$ \\
    \addlinespace[2pt]
    \textbf{PonderNet} & Acc. [\%] & $88.95{\pm}0.67$ & $85.84{\pm}0.46$ & $\mathbf{77.34{\pm}0.89}$ & $82.75{\pm}1.07$ & $80.52{\pm}0.87$ \\
     & ECE & $0.049{\pm}0.006$ & $\mathbf{0.018{\pm}0.005}$ & $0.195{\pm}0.012$ & $0.087{\pm}0.008$ & $0.068{\pm}0.014$ \\
    \bottomrule
  \end{tabular}
  \caption{\textbf{Test accuracy and calibration on the non-MedMNIST datasets.} Entries are mean $\pm$ 95\% confidence interval over five seeds. All checkpoints are selected by the lowest validation loss. Stochastic RIC-SPO pools ten test passes; deterministic RIC-SPO uses one mean-policy pass. \textbf{Bold:} best mean and any method whose reported interval overlaps it.}
  \label{tab:non-medmnist-acc-ece-by-method}
\end{table}

\begin{table}[b]
  \centering
  \footnotesize
  \setlength{\tabcolsep}{3.0pt}
  \renewcommand{\arraystretch}{1.08}
  \begin{tabular}{@{}ll*{5}{c}@{}}
    \toprule
    Method & Metric & \shortstack{Organ\\MNIST} & \shortstack{Derma\\MNIST} & \shortstack{OCT\\MNIST} & \shortstack{Tissue\\MNIST} & \shortstack{Pneumonia\\MNIST} \\
    \midrule
    \textbf{SL} & Acc. [\%] & $80.04{\pm}0.41$ & $76.30{\pm}0.17$ & $\mathbf{86.22{\pm}2.85}$ & $\mathbf{73.93{\pm}0.09}$ & $\mathbf{88.56{\pm}2.18}$ \\
     & ECE & $0.125{\pm}0.004$ & $0.046{\pm}0.007$ & $\mathbf{0.084{\pm}0.022}$ & $\mathbf{0.049{\pm}0.008}$ & $\mathbf{0.091{\pm}0.021}$ \\
    \addlinespace[2pt]
    \textbf{EDL} & Acc. [\%] & $\mathbf{81.76{\pm}0.32}$ & $75.98{\pm}0.71$ & $\mathbf{85.64{\pm}2.21}$ & $68.82{\pm}4.95$ & $\mathbf{89.23{\pm}0.98}$ \\
     & ECE & $0.099{\pm}0.004$ & $0.105{\pm}0.015$ & $0.091{\pm}0.013$ & $\mathbf{0.090{\pm}0.055}$ & $\mathbf{0.077{\pm}0.014}$ \\
    \addlinespace[2pt]
    \textbf{RIC-det.} & Acc. [\%] & $\mathbf{81.23{\pm}1.63}$ & $\mathbf{78.39{\pm}0.58}$ & $\mathbf{85.86{\pm}1.42}$ & $72.30{\pm}0.13$ & $\mathbf{89.65{\pm}1.89}$ \\
     & ECE & $\mathbf{0.050{\pm}0.020}$ & $0.057{\pm}0.009$ & $0.075{\pm}0.009$ & $0.069{\pm}0.011$ & $\mathbf{0.079{\pm}0.014}$ \\
    \addlinespace[2pt]
    \textbf{RIC-stoch.} & Acc. [\%] & $80.95{\pm}0.30$ & $75.82{\pm}0.80$ & $\mathbf{85.39{\pm}1.41}$ & $69.06{\pm}0.40$ & $\mathbf{88.63{\pm}1.80}$ \\
     & ECE & $\mathbf{0.054{\pm}0.012}$ & $\mathbf{0.023{\pm}0.004}$ & $\mathbf{0.053{\pm}0.009}$ & $\mathbf{0.041{\pm}0.006}$ & $\mathbf{0.072{\pm}0.020}$ \\
    \addlinespace[2pt]
    \textbf{ACT} & Acc. [\%] & $74.20{\pm}0.82$ & $70.88{\pm}0.62$ & $73.32{\pm}2.49$ & $66.97{\pm}0.47$ & $\mathbf{85.74{\pm}2.15}$ \\
     & ECE & $0.146{\pm}0.024$ & $0.038{\pm}0.008$ & $0.168{\pm}0.040$ & $\mathbf{0.046{\pm}0.015}$ & $0.122{\pm}0.025$ \\
    \addlinespace[2pt]
    \textbf{PonderNet} & Acc. [\%] & $72.89{\pm}0.98$ & $71.45{\pm}0.37$ & $75.72{\pm}1.57$ & $66.91{\pm}0.55$ & $85.32{\pm}1.98$ \\
     & ECE & $0.138{\pm}0.021$ & $0.034{\pm}0.007$ & $0.149{\pm}0.007$ & $\mathbf{0.034{\pm}0.017}$ & $0.132{\pm}0.025$ \\
    \bottomrule
  \end{tabular}
  \caption{\textbf{Test accuracy and calibration across MedMNIST.} Entries are mean $\pm$ 95\% confidence interval over five seeds. All checkpoints are selected by the lowest validation loss. Stochastic RIC-SPO pools ten test passes; deterministic RIC-SPO uses one mean-policy pass. \textbf{Bold:} best mean and any method whose reported interval overlaps it.}
  \label{tab:medmnist-acc-ece-by-method}
\end{table}

\looseness=-1 Figure~\ref{fig:calibration} compares the final predictive distributions on the test set.
SL models produce a sharply peaked confidence distribution, concentrating predictions near extreme confidence values.
Furthermore, the accuracy of the SL model is substantially lower than its mean confidence.
As a consequence, the corresponding reliability diagram reveals systematic overconfidence across all bins.
In contrast, RIC-SPO produces a broader confidence distribution and is mildly underconfident across most bins. Accuracy and confidence are much better aligned. Cross-entropy sharpens early and stays sharp, driving probabilities toward the boundary. RIC-SPO sharpens while refinement still improves prediction error, and its test distribution stays broader and better calibrated.

Table~\ref{tab:non-medmnist-acc-ece-by-method} and Table~\ref{tab:medmnist-acc-ece-by-method} summarize accuracy and calibration across all datasets, while Figure~\ref{fig:wrapped} shows the corresponding Pareto front under validation loss selection.
Under this selector, the two RIC-SPO evaluation modes expose a clear trade-off. Deterministic mean-policy rollouts have the strongest accuracy profile, while stochastic evaluation lowers ECE on eight of the ten datasets.

\begin{wrapfigure}[22]{r}{0.40\textwidth}
    \centering
    \includegraphics[width=0.40\textwidth]{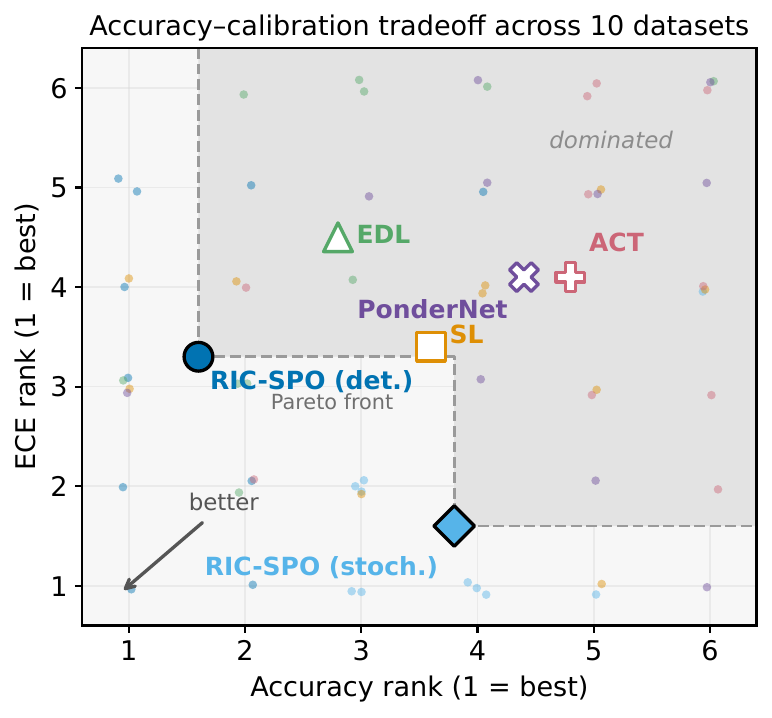}
    \caption{\textbf{Mean accuracy and ECE ranks across ten datasets with validation loss selection.} The two RIC-SPO readouts form the Pareto front. A lower position is better.}
    \label{fig:wrapped}
\end{wrapfigure}

The two deployments occupy different operating points on the most difficult datasets. On MedMNIST, where reliable confidence matters most, the deterministic readout takes the best or tied accuracy on four of the five subsets and the stochastic readout the best or tied ECE on all five. On fine-grained ImageWoof, deterministic RIC-SPO raises accuracy and more than halves ECE at the same time, while stochastic evaluation lowers ECE further. On CIFAR-10N Worst, stochastic evaluation reduces ECE from $0.284$ to $0.051$, while deterministic evaluation retains roughly nine additional accuracy points.
Appendix~\ref{app:selector-sensitivity} examines validation accuracy checkpoint selection, under which the accuracy advantage of deterministic RIC-SPO remains comparatively stable, while its calibration performance deteriorates on several datasets.

ACT and PonderNet track SL's ECE across datasets, so adaptive recurrence alone does not calibrate.
PonderNet is the closest recurrent supervised baseline in our suite. It trains per-step cross-entropy under learned halting weights, with its geometric prior matched to RIC's discount. Its ECE nevertheless remains similar to SL. EDL shares the Dirichlet output family and improves on SL only modestly.
These comparisons control for recurrent refinement, per-step supervision with a matched geometric horizon, and the Dirichlet output family. What remains specific to RIC-SPO is learning the refinements through stochastic returns and policy optimization, which is also where its training dynamics differ from the baselines.

\subsection{\texorpdfstring{Value-Based Halting}{Value-Based Halting}}
\label{sec:halting}

\begin{figure}[t]
  \centering
  \includegraphics[width=\textwidth]{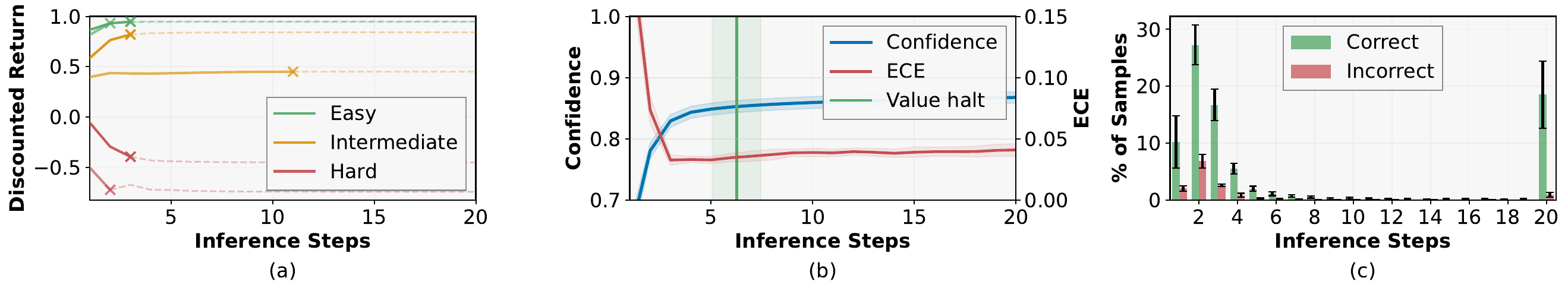}
  \caption{
\textbf{Adaptive computation analysis on ImageWoof (test).}
(a) Discounted return divided by $\log K$, grouped by supervised-consensus difficulty; markers show the halting step.
 On easy inputs the policy takes confident predictions early and the value function imposes early stopping. On intermediate examples the policy takes more compute and gradually increases return. On hard images the value function stops early when return is about to decrease.
(b) Average confidence and ECE as a function of inference steps. The vertical line marks the mean value-based stopping step.
(c) Distribution of stopping steps for correctly and incorrectly classified inputs.
}
\label{fig:calibration2}
\end{figure}

Proposition~\ref{prop:value} motivates value-based stopping. Appendix~\ref{app:critic-validation} validates the learned critic against vine estimates of the future return for both deployments. We assess this mechanism for deterministic RIC-SPO on the ImageWoof test set.
Figure~\ref{fig:calibration2}(a) shows the stopping behavior on examples grouped by supervised model consensus on the ImageWoof test set.
We evaluate five independently trained SL ResNet-34 models on each test image and group by consensus. 
Images classified correctly by all five models are labeled Easy, images classified correctly by more than half but not all models are labeled Intermediate, and images classified correctly by at most half of the models are labeled Hard.
Easy inputs saturate quickly, intermediate inputs benefit from additional refinement steps, and on hard images the policy stops before additional negative reward accrues.

Figure~\ref{fig:calibration2}(b) shows confidence rising across refinement steps while ECE falls.
At early steps the agent stays cautious and sharpens its belief as more computation is used.
As confidence stabilizes, predicted future return approaches zero, and the value rule stops inference once further refinement yields marginal gains.
Figure~\ref{fig:calibration2}(c) confirms that the critic allocates more steps to inputs that are ultimately classified correctly, and stops early on hard images that are misclassified.
Some inputs reach the $20$-step horizon when the critic overestimates the remaining return. Together, these behaviors follow Proposition~\ref{prop:value}, with refinement continuing while it lowers prediction error and stopping once the expected improvement vanishes.

\section{Related Work}
\label{sec:rw}
RIC sits at the intersection of several strands of work that have studied learning from evaluative feedback, adaptive computation, classification, deliberation, and calibration. 
We organize the related work around the differences to RIC as a formulation in which the class distribution itself is the action, reward measures its stepwise improvement, and the critic decides when refinement should stop.

\noindent\textbf{Reward-driven learning.}
RIC draws on the reward-driven view of learning as improvement from evaluative feedback \citep{sutton2018reinforcement}.
Early machine learning already tied learning to improvement through experience, as in Samuel's checkers program \citep{samuel1959some}.
Temporal difference learning later trained predictions from differences between successive predictions \citep{sutton1988learning}.
The supervised error correction line, including Widrow and Hoff's adaptive switching circuits, trains from desired outputs by reducing error \citep{widrow1960adaptive}.
RIC connects these histories inside a classifier. The class distribution is the action, and reward measures prediction improvement.

\looseness=-1 \noindent\textbf{RL control for classification.}
RL control methods are closer in training signal, but their actions control classifier computation or make discrete classification decisions.
One group uses RL to decide how a supervised classifier computes. These policies select attention glimpses \citep{mnih2014recurrent}, choose which layers or residual blocks to execute \citep{wang2018skipnet,wu2018blockdrop}, or acquire features until a discrete classify-now action \citep{dulacarnold2011datum,janisch2019classification}.
A second group makes classification itself the decision. Actions edit feature buffers whose state values give the class \citep{wiering2011reinforcement}, move agents whose consensus readout receives a terminal loss \citep{mousavi2019multi}, or emit one hard label per sample under class-weighted accuracy rewards \citep{lin2020deep}.
Across this line, the action is an attention location, a subnetwork, a feature, or a hard label, and the reward is accuracy or compute cost.
To our knowledge, RIC is the first formulation in which the action is the class distribution itself, the reward is the stepwise improvement of a strictly proper score of that action, and the learned value of remaining improvement decides when refinement stops.

\noindent\textbf{Adaptive computation.}
Adaptive computation methods share RIC's goal of spending compute according to input difficulty, usually through supervised halting machinery.
\citet{graves2016adaptive} introduces Adaptive Computation Time, which trades task loss against an explicit ponder cost, and \citet{banino2021pondernet} learn a distribution over stopping times regularized toward a geometric prior.
Recurrent refinement with shared weights also appears in the Universal Transformer \citep{dehghani2018universal}, and spatially adaptive depth extends this idea to convolutional residual networks \citep{figurnov2017spatially}.
Early-exit networks such as BranchyNet \citep{teerapittayanon2016branchynet} attach side classifiers to a backbone and stop once predictive entropy falls below a threshold, a criterion that cannot separate label ambiguity from reducible error.
RIC uses the critic as the stopping signal. Proposition~\ref{prop:value} gives this rule the classical value of computation interpretation for monitoring anytime algorithms \citep{russell1991principles,hansen2001monitoring}.

\noindent\textbf{Deliberation and recurrent decision making.}
Recent work treats deliberation and computation control as part of the action space.
\citet{hanna2025thinking} formalize deliberation as mental actions in a thought MDP, and \citet{orenstein2025toward} give agents explicit compute actions and feedback about compute cost.
Latent reasoning in language models studies both the medium and duration of computation before answering \citep{hao2025training,ning2025learning}.
Together, these works show that reasoning can move into a latent space and that reward can choose how long it continues.
Related findings also support the recurrent design. Recurrent agents trained by RL can develop speed--accuracy trade-offs and changes of mind \citep{wispinski2026primate}, and recurrent value networks can learn hidden states correlated with Bayesian beliefs \citep{lambrechts2022recurrent}.
RIC places the action directly on the class prediction.

\noindent\textbf{Calibration and uncertainty.}
Calibration provides another reason to pursue iterative belief refinement.
Cross-entropy-trained modern networks are often overconfident \citep{guo2017calibration}, a behavior \citet{wei2022mitigating} connect to continued logit-norm growth during training.
A common approach is post-hoc adjustment such as temperature scaling \citep{guo2017calibration}.
Training rewards shape predicted uncertainty during learning itself \citep{damani2025beyond,baniharouni2026rewarding}. RIC follows this route.
RIC shares its Dirichlet output family with evidential deep learning, which predicts one input-conditioned Dirichlet in a single supervised forward pass to represent uncertainty \citep{sensoy2018evidential}, while RIC samples and refines actions through an RL objective. \citet{juergens2024epistemic} identify limitations in the faithfulness of epistemic uncertainty estimates from evidential methods, including problems of identifiability and convergence.

\section{Conclusion}

\label{sec:conc}

In this work, we introduced Reinforced Iterative Classification, a formulation of classification as reward-driven belief refinement. In the deterministic limit, RIC recovers the probability target of cross-entropy while learning it through stochastic returns. RIC-SPO exposes an accuracy--calibration trade-off at deployment: deterministic mean-policy rollouts are generally more accurate, while stochastic rollouts are generally better calibrated and may lose accuracy. The critic estimates the value of further refinement, providing an adaptive stopping rule without a separate halting module. More importantly, RIC and supervised learning follow fundamentally different training dynamics. RIC-SPO fits the training data more slowly, yet reaches higher validation accuracy than supervised learning at matched training accuracy. Supervised learning couples accuracy with rapid and often extreme certainty. RIC instead learns from the returns earned by stochastic refinements, causing confidence to lag behind accuracy. This separates the evolution of accuracy and confidence, although the resulting endpoint depends on the dataset and evaluation mode.

\medskip
\noindent\textbf{Limitations and Prospects.}
Relative to supervised learning, RIC-SPO requires an order of magnitude more training epochs, which translates into roughly ten times slower training. Its slow convergence may leave much of RIC's potential unrealized through exposure to overfitting during prolonged training. We believe better reinforcement learning algorithms could improve both efficiency and final performance while preserving the training dynamics observed herein. Another noteworthy limitation is that the Dirichlet distribution becomes unstable near the simplex boundary, especially as the number of classes grows. This limits our current validation to relatively small label spaces. Scaling RIC will require more stable ways of navigating the simplex. Temperature scaling and label smoothing are orthogonal to RIC, and we leave combining them with it to future work. Overall, we believe RIC is a first step towards learning systems that improve their predictions through reinforcement rather than imitation.

\newpage
\bibliography{main}

\begin{thebibliography}{45}
\providecommand{\natexlab}[1]{#1}
\providecommand{\url}[1]{\texttt{#1}}
\expandafter\ifx\csname urlstyle\endcsname\relax
  \providecommand{\doi}[1]{doi: #1}\else
  \providecommand{\doi}{doi: \begingroup \urlstyle{rm}\Url}\fi

\bibitem[Bani-Harouni et~al.(2026)Bani-Harouni, Pellegrini, Stangel, Özsoy,
  Zaripova, Navab, and Keicher]{baniharouni2026rewarding}
David Bani-Harouni, Chantal Pellegrini, Paul Stangel, Ege Özsoy, Kamilia
  Zaripova, Nassir Navab, and Matthias Keicher.
\newblock Rewarding doubt: A reinforcement learning approach to calibrated
  confidence expression of large language models.
\newblock In \emph{International Conference on Learning Representations}, 2026.
\newblock URL \url{https://openreview.net/forum?id=yResLmrVO1}.

\bibitem[Banino et~al.(2021)Banino, Balaguer, and
  Blundell]{banino2021pondernet}
Andrea Banino, Jan Balaguer, and Charles Blundell.
\newblock {PonderNet}: Learning to ponder.
\newblock \emph{arXiv preprint arXiv:2107.05407}, 2021.

\bibitem[Barto(2019)]{barto2019reinforcement}
Andrew~G. Barto.
\newblock Reinforcement learning: Connections, surprises, challenges.
\newblock \emph{AI Magazine}, 40\penalty0 (1):\penalty0 3--15, 2019.
\newblock \doi{10.1609/aimag.v40i1.2844}.

\bibitem[Damani et~al.(2026)Damani, Puri, Slocum, Shenfeld, Choshen, Kim, and
  Andreas]{damani2025beyond}
Mehul Damani, Isha Puri, Stewart Slocum, Idan Shenfeld, Leshem Choshen, Yoon
  Kim, and Jacob Andreas.
\newblock Beyond binary rewards: Training {LM}s to reason about their
  uncertainty.
\newblock In \emph{International Conference on Learning Representations}, 2026.
\newblock URL \url{https://openreview.net/forum?id=ASQ649zdHm}.

\bibitem[Dehghani et~al.(2019)Dehghani, Gouws, Vinyals, Uszkoreit, and
  Kaiser]{dehghani2018universal}
Mostafa Dehghani, Stephan Gouws, Oriol Vinyals, Jakob Uszkoreit, and {\L}ukasz
  Kaiser.
\newblock Universal transformers.
\newblock In \emph{International Conference on Learning Representations}, 2019.
\newblock URL \url{https://openreview.net/forum?id=HyzdRiR9Y7}.

\bibitem[Deng et~al.(2009)Deng, Dong, Socher, Li, Li, and
  Fei-Fei]{deng2009imagenet}
Jia Deng, Wei Dong, Richard Socher, Li-Jia Li, Kai Li, and Li~Fei-Fei.
\newblock {ImageNet}: A large-scale hierarchical image database.
\newblock In \emph{2009 IEEE Conference on Computer Vision and Pattern
  Recognition (CVPR)}, pp.\  248--255, 2009.

\bibitem[Dulac-Arnold et~al.(2011)Dulac-Arnold, Denoyer, Preux, and
  Gallinari]{dulacarnold2011datum}
Gabriel Dulac-Arnold, Ludovic Denoyer, Philippe Preux, and Patrick Gallinari.
\newblock Datum-wise classification: a sequential approach to sparsity.
\newblock In \emph{Joint European Conference on Machine Learning and Knowledge
  Discovery in Databases (ECML PKDD)}, pp.\  375--390, 2011.

\bibitem[Figurnov et~al.(2017)Figurnov, Collins, Zhu, Zhang, Huang, Vetrov, and
  Salakhutdinov]{figurnov2017spatially}
Michael Figurnov, Maxwell~D Collins, Yukun Zhu, Li~Zhang, Jonathan Huang,
  Dmitry Vetrov, and Ruslan Salakhutdinov.
\newblock Spatially adaptive computation time for residual networks.
\newblock In \emph{Proceedings of the IEEE Conference on Computer Vision and
  Pattern Recognition (CVPR)}, pp.\  1790--1799, 2017.

\bibitem[Gneiting \& Raftery(2007)Gneiting and Raftery]{gneiting2007strictly}
Tilmann Gneiting and Adrian~E. Raftery.
\newblock Strictly proper scoring rules, prediction, and estimation.
\newblock \emph{Journal of the American Statistical Association}, 102\penalty0
  (477):\penalty0 359--378, 2007.
\newblock \doi{10.1198/016214506000001437}.

\bibitem[Graves(2016)]{graves2016adaptive}
Alex Graves.
\newblock Adaptive computation time for recurrent neural networks.
\newblock \emph{arXiv preprint arXiv:1603.08983}, 2016.

\bibitem[Guo et~al.(2017)Guo, Pleiss, Sun, and Weinberger]{guo2017calibration}
Chuan Guo, Geoff Pleiss, Yu~Sun, and Kilian~Q Weinberger.
\newblock On calibration of modern neural networks.
\newblock In \emph{Proceedings of the 34th International Conference on Machine
  Learning}, volume~70 of \emph{Proceedings of Machine Learning Research}, pp.\
   1321--1330. PMLR, 2017.

\bibitem[Hanna \& Corrado(2025)Hanna and Corrado]{hanna2025thinking}
Josiah~P. Hanna and Nicholas~E. Corrado.
\newblock When can model-free reinforcement learning be enough for thinking?
\newblock In \emph{Advances in Neural Information Processing Systems},
  volume~38, 2025.
\newblock URL
  \url{https://proceedings.neurips.cc/paper_files/paper/2025/hash/2a4179ef39846557e99f6bfac580ea2e-Abstract-Conference.html}.

\bibitem[Hansen \& Zilberstein(2001)Hansen and
  Zilberstein]{hansen2001monitoring}
Eric~A. Hansen and Shlomo Zilberstein.
\newblock Monitoring and control of anytime algorithms: A dynamic programming
  approach.
\newblock \emph{Artificial Intelligence}, 126\penalty0 (1--2):\penalty0
  139--157, 2001.

\bibitem[Hao et~al.(2025)Hao, Sukhbaatar, Su, Li, Hu, Weston, and
  Tian]{hao2025training}
Shibo Hao, Sainbayar Sukhbaatar, DiJia Su, Xian Li, Zhiting Hu, Jason Weston,
  and Yuandong Tian.
\newblock Training large language models to reason in a continuous latent
  space.
\newblock In \emph{Conference on Language Modeling (COLM)}, 2025.

\bibitem[He et~al.(2016)He, Zhang, Ren, and Sun]{he2016deep}
Kaiming He, Xiangyu Zhang, Shaoqing Ren, and Jian Sun.
\newblock Deep residual learning for image recognition.
\newblock In \emph{Proceedings of the IEEE Conference on Computer Vision and
  Pattern Recognition (CVPR)}, pp.\  770--778, 2016.

\bibitem[Howard(2019)]{howard2020imagenette}
Jeremy Howard.
\newblock {Imagenette} and {Imagewoof} datasets.
\newblock \url{https://github.com/fastai/imagenette}, 2019.
\newblock Accessed: 2026-03-03.

\bibitem[Janisch et~al.(2019)Janisch, Pevn{\'y}, and
  Lis{\'y}]{janisch2019classification}
Jarom{\'i}r Janisch, Tom{\'a}{\v{s}} Pevn{\'y}, and Viliam Lis{\'y}.
\newblock Classification with costly features using deep reinforcement
  learning.
\newblock In \emph{Proceedings of the AAAI Conference on Artificial
  Intelligence}, volume~33, pp.\  3959--3966, 2019.

\bibitem[Juergens et~al.(2024)Juergens, Meinert, Bengs, H{\"u}llermeier, and
  Waegeman]{juergens2024epistemic}
Mira Juergens, Nis Meinert, Viktor Bengs, Eyke H{\"u}llermeier, and Willem
  Waegeman.
\newblock Is epistemic uncertainty faithfully represented by evidential deep
  learning methods?
\newblock In \emph{Proceedings of the 41st International Conference on Machine
  Learning}, volume 235 of \emph{Proceedings of Machine Learning Research},
  pp.\  22624--22642. PMLR, 2024.
\newblock URL \url{https://proceedings.mlr.press/v235/juergens24a.html}.

\bibitem[Krizhevsky(2009)]{krizhevsky2009learning}
Alex Krizhevsky.
\newblock Learning multiple layers of features from tiny images.
\newblock Technical report, University of Toronto, 2009.
\newblock URL
  \url{https://www.cs.toronto.edu/~kriz/learning-features-2009-TR.pdf}.

\bibitem[Lambrechts et~al.(2022)Lambrechts, Bolland, and
  Ernst]{lambrechts2022recurrent}
Gaspard Lambrechts, Adrien Bolland, and Damien Ernst.
\newblock Recurrent networks, hidden states and beliefs in partially observable
  environments.
\newblock \emph{Transactions on Machine Learning Research}, 2022.
\newblock URL \url{https://openreview.net/forum?id=dkHfV3wB2l}.

\bibitem[Lee et~al.(2015)Lee, Xie, Gallagher, Zhang, and Tu]{lee2015deeply}
Chen-Yu Lee, Saining Xie, Patrick Gallagher, Zhengyou Zhang, and Zhuowen Tu.
\newblock Deeply-supervised nets.
\newblock In \emph{Proceedings of the 18th International Conference on
  Artificial Intelligence and Statistics}, volume~38 of \emph{Proceedings of
  Machine Learning Research}, pp.\  562--570. PMLR, 2015.

\bibitem[Leung \& Bovy(2024)Leung and Bovy]{leung2024galaxy10}
W.~Henry Leung and Jo~Bovy.
\newblock {Galaxy10 DECaLS}.
\newblock Zenodo, 2024.
\newblock URL \url{https://doi.org/10.5281/zenodo.10845026}.

\bibitem[Lin et~al.(2020)Lin, Chen, and Qi]{lin2020deep}
Enlu Lin, Qiong Chen, and Xiaoming Qi.
\newblock Deep reinforcement learning for imbalanced classification.
\newblock \emph{Applied Intelligence}, 50:\penalty0 2488--2502, 2020.

\bibitem[Mnih et~al.(2014)Mnih, Heess, Graves, and
  Kavukcuoglu]{mnih2014recurrent}
Volodymyr Mnih, Nicolas Heess, Alex Graves, and Koray Kavukcuoglu.
\newblock Recurrent models of visual attention.
\newblock In \emph{Advances in Neural Information Processing Systems},
  volume~27, pp.\  2204--2212, 2014.

\bibitem[Mousavi et~al.(2019)Mousavi, Nazari, Tak{\'a}{\v{c}}, and
  Motee]{mousavi2019multi}
Hossein~K Mousavi, Mohammadreza Nazari, Martin Tak{\'a}{\v{c}}, and Nader
  Motee.
\newblock Multi-agent image classification via reinforcement learning.
\newblock In \emph{2019 IEEE/RSJ International Conference on Intelligent Robots
  and Systems (IROS)}, pp.\  5020--5027. IEEE, 2019.

\bibitem[Ning et~al.(2025)Ning, Kuo, and Gomes]{ning2025learning}
Alex Ning, Yen-Ling Kuo, and Gabe Gomes.
\newblock Learning when to stop: Adaptive latent reasoning via reinforcement
  learning.
\newblock \emph{arXiv preprint arXiv:2511.21581}, 2025.

\bibitem[Orenstein et~al.(2026)Orenstein, Chen, Santos, Sapara, and
  Bowling]{orenstein2025toward}
Adrian Orenstein, Jessica Chen, Gwyneth Anne~Delos Santos, Bayley Sapara, and
  Michael Bowling.
\newblock Toward agents that reason about their computation.
\newblock \emph{Reinforcement Learning Journal}, 7, 2026.

\bibitem[Puterman(1994)]{puterman1994markov}
Martin~L. Puterman.
\newblock \emph{Markov Decision Processes: Discrete Stochastic Dynamic
  Programming}.
\newblock John Wiley \& Sons, Inc., USA, 1st edition, 1994.
\newblock ISBN 0471619779.

\bibitem[Russell \& Wefald(1991)Russell and Wefald]{russell1991principles}
Stuart Russell and Eric Wefald.
\newblock Principles of metareasoning.
\newblock \emph{Artificial Intelligence}, 49\penalty0 (1--3):\penalty0
  361--395, 1991.

\bibitem[Samuel(1959)]{samuel1959some}
Arthur~L. Samuel.
\newblock Some studies in machine learning using the game of checkers.
\newblock \emph{IBM Journal of Research and Development}, 3\penalty0
  (3):\penalty0 210--229, 1959.

\bibitem[Schulman et~al.(2016)Schulman, Moritz, Levine, Jordan, and
  Abbeel]{schulman2016high}
John Schulman, Philipp Moritz, Sergey Levine, Michael~I. Jordan, and Pieter
  Abbeel.
\newblock High-dimensional continuous control using generalized advantage
  estimation.
\newblock In \emph{International Conference on Learning Representations}, 2016.

\bibitem[Schulman et~al.(2017)Schulman, Wolski, Dhariwal, Radford, and
  Klimov]{schulman2017proximal}
John Schulman, Filip Wolski, Prafulla Dhariwal, Alec Radford, and Oleg Klimov.
\newblock Proximal policy optimization algorithms.
\newblock \emph{arXiv preprint arXiv:1707.06347}, 2017.

\bibitem[Sensoy et~al.(2018)Sensoy, Kaplan, and Kandemir]{sensoy2018evidential}
Murat Sensoy, Lance Kaplan, and Melih Kandemir.
\newblock Evidential deep learning to quantify classification uncertainty.
\newblock In \emph{Advances in Neural Information Processing Systems},
  volume~31, 2018.

\bibitem[Sutton(1988)]{sutton1988learning}
Richard~S. Sutton.
\newblock Learning to predict by the methods of temporal differences.
\newblock \emph{Machine Learning}, 3:\penalty0 9--44, 1988.

\bibitem[Sutton \& Barto(2018)Sutton and Barto]{sutton2018reinforcement}
Richard~S. Sutton and Andrew~G. Barto.
\newblock \emph{Reinforcement Learning: An Introduction}.
\newblock MIT Press, Cambridge, MA, second edition, 2018.

\bibitem[Teerapittayanon et~al.(2016)Teerapittayanon, McDanel, and
  Kung]{teerapittayanon2016branchynet}
Surat Teerapittayanon, Bradley McDanel, and H.~T. Kung.
\newblock {BranchyNet}: Fast inference via early exiting from deep neural
  networks.
\newblock In \emph{23rd International Conference on Pattern Recognition
  (ICPR)}, pp.\  2464--2469. IEEE, 2016.

\bibitem[Wang et~al.(2018)Wang, Yu, Dou, Darrell, and
  Gonzalez]{wang2018skipnet}
Xin Wang, Fisher Yu, Zi-Yi Dou, Trevor Darrell, and Joseph~E Gonzalez.
\newblock {SkipNet}: Learning dynamic routing in convolutional networks.
\newblock In \emph{Proceedings of the European Conference on Computer Vision
  (ECCV)}, pp.\  420--436, 2018.

\bibitem[Wei et~al.(2022{\natexlab{a}})Wei, Xie, Cheng, Feng, An, and
  Li]{wei2022mitigating}
Hongxin Wei, Renchunzi Xie, Hao Cheng, Lei Feng, Bo~An, and Yixuan Li.
\newblock Mitigating neural network overconfidence with logit normalization.
\newblock In \emph{Proceedings of the 39th International Conference on Machine
  Learning}, volume 162 of \emph{Proceedings of Machine Learning Research},
  pp.\  23631--23644. PMLR, 2022{\natexlab{a}}.

\bibitem[Wei et~al.(2022{\natexlab{b}})Wei, Zhu, Cheng, Liu, Niu, and
  Liu]{wei2021learning}
Jiaheng Wei, Zhaowei Zhu, Hao Cheng, Tongliang Liu, Gang Niu, and Yang Liu.
\newblock Learning with noisy labels revisited: A study using real-world human
  annotations.
\newblock In \emph{International Conference on Learning Representations},
  2022{\natexlab{b}}.
\newblock URL \url{https://openreview.net/forum?id=TBWA6PLJZQm}.

\bibitem[Widrow \& Hoff(1960)Widrow and Hoff]{widrow1960adaptive}
Bernard Widrow and Marcian~E. Hoff, Jr.
\newblock Adaptive switching circuits.
\newblock In \emph{1960 IRE WESCON Convention Record}, volume~4, pp.\  96--104,
  1960.

\bibitem[Wiering et~al.(2011)Wiering, {van Hasselt}, Pietersma, and
  Schomaker]{wiering2011reinforcement}
Marco~A Wiering, Hado {van Hasselt}, Auke-Dirk Pietersma, and Lambert
  Schomaker.
\newblock Reinforcement learning algorithms for solving classification
  problems.
\newblock In \emph{2011 IEEE Symposium on Adaptive Dynamic Programming and
  Reinforcement Learning (ADPRL)}, pp.\  91--96. IEEE, 2011.

\bibitem[Wispinski et~al.(2026)Wispinski, Stone, Singhal, Pilarski, and
  Chapman]{wispinski2026primate}
Nathan~J Wispinski, Scott~A Stone, Anthony Singhal, Patrick~M Pilarski, and
  Craig~S Chapman.
\newblock Primate-like perceptual decision making emerges through deep
  recurrent reinforcement learning.
\newblock \emph{arXiv preprint arXiv:2601.12577}, 2026.

\bibitem[Wu et~al.(2018)Wu, Nagarajan, Kumar, Rennie, Davis, Grauman, and
  Feris]{wu2018blockdrop}
Zuxuan Wu, Tushar Nagarajan, Abhishek Kumar, Steven Rennie, Larry~S Davis,
  Kristen Grauman, and Rogerio Feris.
\newblock {BlockDrop}: Dynamic inference paths in residual networks.
\newblock In \emph{Proceedings of the IEEE Conference on Computer Vision and
  Pattern Recognition (CVPR)}, pp.\  8817--8826, 2018.

\bibitem[Xie et~al.(2025)Xie, Zhang, Yang, Hutter, and Xu]{xie2025simple}
Zhengpeng Xie, Qiang Zhang, Fan Yang, Marco Hutter, and Renjing Xu.
\newblock Simple policy optimization.
\newblock In Aarti Singh, Maryam Fazel, Daniel Hsu, Simon Lacoste-Julien, Felix
  Berkenkamp, Tegan Maharaj, Kiri Wagstaff, and Jerry Zhu (eds.),
  \emph{Proceedings of the 42nd International Conference on Machine Learning},
  volume 267 of \emph{Proceedings of Machine Learning Research}, pp.\
  68813--68824. PMLR, 13--19 Jul 2025.
\newblock URL \url{https://proceedings.mlr.press/v267/xie25m.html}.

\bibitem[Yang et~al.(2023)Yang, Shi, Wei, Liu, Zhao, Ke, Pfister, and
  Ni]{yang2023medmnist}
Jiancheng Yang, Rui Shi, Donglai Wei, Zequan Liu, Lin Zhao, Bilian Ke,
  Hanspeter Pfister, and Bingbing Ni.
\newblock {MedMNIST} v2: A large-scale lightweight benchmark for {2D} and {3D}
  biomedical image classification.
\newblock \emph{Scientific Data}, 10\penalty0 (1):\penalty0 41, 2023.
\newblock \doi{10.1038/s41597-022-01721-8}.

\end{thebibliography}
\bibliographystyle{tmlr}

\newpage
\appendix
\setcounter{topnumber}{2}
\setcounter{totalnumber}{3}
\section{Theoretical Supplement}
\label{sec:theory_supp}

\makeatletter
\setlength{\@fptop}{0pt}
\setlength{\@fpsep}{18pt}
\setlength{\@fpbot}{0pt plus 1fil}
\makeatother

\subsection{\texorpdfstring{Setup and Assumptions}{Setup and Assumptions}}
\label{sec:assumptions}
The proofs use two regularity conditions.

\begin{enumerate}
    \item \textbf{Interior support.} The policy is supported on the interior of the simplex, so $a_k > 0$ on the policy support and the label log probability $\log a_y$ is finite.
    \item \textbf{Integrability.} $\E_\pi[\sum_{t=1}^{\infty} \gamma^{t-1} |\log a_{t,y}|] < \infty$, ensuring that the geometric mixture objective and all sum interchange steps are valid.
\end{enumerate}

The numerical offset and bounded training concentration give RIC-SPO finite Dirichlet log moments. Proposition~\ref{prop:target} concerns the unrestricted family: the Dirichlet variance vanishes as $c\to\infty$, and boundary means are approached from the simplex interior.

Fix $x$ and write $q=\TrueCond$ and $L_t=\mathrm{KL}(q\|a_t)$. Since the policy and transitions do not observe $y$, the trajectory is conditionally independent of $y$ given $x$, and $y\mid s_t\sim q$. We use
\begin{equation}
\E_{y\sim q}[-\log a_y]=H(q)+\mathrm{KL}(q\|a).
\label{eq:log_score}
\end{equation}
The integrability condition justifies the exchanges of sums and expectations below.

\subsection{\texorpdfstring{Proof of Lemma~\ref{lem:jensen_gap} (Cross-Entropy Decomposition for Stochastic Predictions)}{Proof of Lemma 2 (Cross-Entropy Decomposition for Stochastic Predictions)}}
\label{sec:jensen_gap_proof}

\begin{proof}
Fix $s$ and let $\mu=\E_\pi[a\mid s]$. Expanding $\Breg$ and averaging over the action gives
\[
\begin{aligned}
\Breg(a_y\|\mu_y)
&=-\log a_y+\log\mu_y+\frac{a_y-\mu_y}{\mu_y}, \\
\E_\pi[-\log a_y\mid s]
&=-\log\mu_y-\frac{\E_\pi[a_y\mid s]-\mu_y}{\mu_y}
  +\E_\pi[\Breg(a_y\|\mu_y)\mid s] \\
&=-\log\mu_y+\E_\pi[\Breg(a_y\|\mu_y)\mid s].
\end{aligned}
\]
Averaging over $y\sim q$ gives
\[
\begin{aligned}
\E_{y\sim q,\,a\sim\pi}[-\log a_y\mid s]
&=-\sum_y q_y\log\mu_y
  +\E_{y\sim q,\,a\sim\pi}[\Breg(a_y\|\mu_y)\mid s] \\
&=H(q)+\mathrm{KL}(q\|\mu)
  +\E_{y\sim q,\,a\sim\pi}[\Breg(a_y\|\mu_y)\mid s].
\end{aligned}
\]
This is Equation~\eqref{eq:decomp}. Applying Equation~\eqref{eq:log_score} to each realization of $a$ and subtracting $H(q)$ gives
\begin{equation}
\begin{aligned}
\E_\pi[\mathrm{KL}(q\|a)\mid s]
&=\E_{y\sim q,\,a\sim\pi}[-\log a_y\mid s]-H(q) \\
&=\mathrm{KL}(q\|\mu)
+\E_{y\sim q,\,a\sim\pi}[\Breg(a_y\|\mu_y)\mid s].
\end{aligned}
\label{eq:kl_split}
\end{equation}
Finally, $\Breg(u\|v)\ge0$ and $\Breg(u\|v)=0\iff u=v$, so
\[
\E_{y\sim q,\,a\sim\pi}[\Breg(a_y\|\mu_y)\mid s]=0
\quad\Longleftrightarrow\quad
\Pr_\pi(a_y=\mu_y\mid s)=1\quad\forall y:q_y>0.
\]
Thus the stochastic cost in Equation~\eqref{eq:decomp} is exactly the excess expected cross-entropy over that of the mean prediction.
\end{proof}
\newpage
\subsection{\texorpdfstring{Proof of Proposition~\ref{prop:target} (Optimal Prediction Target)}{Proof of Proposition 1 (Optimal Prediction Target)}}
\label{sec:target_proof}

\begin{proof}
Let $N\sim\Geom(1-\gamma)$ be independent of the trajectory and $w_t=(1-\gamma)\gamma^{t-1}$. Lemma~\ref{lem:geom_discount} and telescoping give
\[
\begin{aligned}
J_{x,y}(\pi)
&=\E_{N,\pi}\!\left[\sum_{t=1}^{N}r_t\,\middle|\,x,y\right] \\
&=\E_{N,\pi}\!\left[\sum_{t=1}^{N}
  (\log a_{t,y}-\log a_{t-1,y})\,\middle|\,x,y\right] \\
&=\E_{N,\pi}[\log a_{N,y}-\log a_{0,y}\mid x,y] \\
&=\log K+\sum_{t=1}^{\infty}w_t\,\E_\pi[\log a_{t,y}\mid x].
\end{aligned}
\]
The last equality uses $a_{0,y}=1/K$ and conditional independence of $a_t$ and $y$. Averaging over $y\sim q$ and applying Equation~\eqref{eq:log_score} gives
\begin{equation}
\begin{aligned}
\E_{y\sim q}[J_{x,y}(\pi)]
&=\log K+\sum_{t=1}^{\infty}w_t\,
  \E_\pi\!\left[\E_{y\sim q}[\log a_{t,y}]\mid x\right] \\
&=\log K-\sum_{t=1}^{\infty}w_t\,
  \bigl(H(q)+\E_\pi[L_t\mid x]\bigr) \\
&=L_0-\sum_{t=1}^{\infty}w_t\,\E_\pi[L_t\mid x]
\le L_0,
\end{aligned}
\label{eq:target_proof_deficit}
\end{equation}
The final equality uses
\[
\sum_{t\ge1}w_t=(1-\gamma)\sum_{t\ge1}\gamma^{t-1}=1,
\qquad L_0=\mathrm{KL}(q\|a_0)=\log K-H(q).
\]
Equation~\eqref{eq:target_proof_deficit} is Equation~\eqref{eq:valueid} at $s_0$.

For $\mu_t=\E_\pi[a_t\mid x]$, Equation~\eqref{eq:kl_split} applies to the marginal law at depth $t$:
\[
\E_\pi[L_t\mid x]
=\mathrm{KL}(q\|\mu_t)
+\E_{y\sim q,\,a_t\sim\pi}[\Breg(a_{t,y}\|\mu_{t,y})\mid x].
\]
Under the approximation assumption, choose $\{\pi_n\}$ with $\sup_t\E_{\pi_n}[L_t\mid x]\to0$. Then
\[
\begin{aligned}
0\le L_0-\E_{y\sim q}[J_{x,y}(\pi_n)]
&=\sum_{t=1}^{\infty}w_t\,\E_{\pi_n}[L_t\mid x] \\
&\le\sup_t\E_{\pi_n}[L_t\mid x]\longrightarrow0.
\end{aligned}
\]
Thus the objective can approach $L_0$. Conversely, suppose
\[
\E_{y\sim q}[J_{x,y}(\pi_n)]\longrightarrow L_0,
\qquad \mu_t^{(n)}=\E_{\pi_n}[a_t\mid x].
\]
For every fixed $t$, Equation~\eqref{eq:target_proof_deficit} and nonnegativity give
\[
\begin{aligned}
0\le\E_{\pi_n}[L_t\mid x]
&\le\frac{L_0-\E_{y\sim q}[J_{x,y}(\pi_n)]}{w_t}
\longrightarrow0,\qquad w_t>0, \\
0\le\mathrm{KL}(q\|\mu_t^{(n)})
&\le\E_{\pi_n}[L_t\mid x]\longrightarrow0, \\
0\le\E_{y\sim q,\,a_t\sim\pi_n}
[\Breg(a_{t,y}\|\mu_{t,y}^{(n)})\mid x]
&\le\E_{\pi_n}[L_t\mid x]\longrightarrow0.
\end{aligned}
\]
By Pinsker's inequality,
\[
\|\mu_t^{(n)}-q\|_1^2
\le2\,\mathrm{KL}(q\|\mu_t^{(n)})\longrightarrow0.
\]
Hence, along any sequence approaching $L_0$, the mean prediction converges to $q$ and the cost of stochastic actions vanishes at every fixed refinement depth.
\end{proof}

\newpage
\subsection{\texorpdfstring{Proof of Proposition~\ref{prop:value} (Value of Refinement)}{Proof of Proposition 2 (Value of Refinement)}}
\label{sec:value_proof}

\begin{proof}
Fix $s_t=(x,a_t,\tau_t)$ and take $N\sim\Geom(1-\gamma)$ independently of the label and trajectory. Applying Lemma~\ref{lem:geom_discount} to the label-averaged return and telescoping gives
\[
\begin{aligned}
\ValuePi(s_t)
&=\E_{y\sim q}\,\E_\pi\!\left[
\sum_{k=1}^{\infty}\gamma^{k-1}r_{t+k}\,\middle|\,s_t,y\right] \\
&=\E_{y\sim q}\,\E_{N,\pi}\!\left[
\sum_{k=1}^{N}r_{t+k}\,\middle|\,s_t,y\right] \\
&=\E_{y\sim q}\,\E_{N,\pi}\!\left[
\sum_{k=1}^{N}(\log a_{t+k,y}-\log a_{t+k-1,y})
\,\middle|\,s_t,y\right] \\
&=\E_{y\sim q}\,\E_{N,\pi}[\log a_{t+N,y}-\log a_{t,y}\mid s_t,y].
\end{aligned}
\]
Since the trajectory is independent of $y$ given $x$, we have $y\mid s_t,a_{t+N}\sim q$. Averaging the label inside the rollout expectation and applying Equation~\eqref{eq:log_score} to both predictions gives
\[
\begin{aligned}
\ValuePi(s_t)
&=\E_{N,\pi}\!\left[
\E_{y\sim q}[\log a_{t+N,y}-\log a_{t,y}]\,\middle|\,s_t\right] \\
&=\E_{N,\pi}\!\left[
\bigl(-H(q)-L_{t+N}\bigr)-\bigl(-H(q)-L_t\bigr)\,\middle|\,s_t\right] \\
&=\E_{N,\pi}[L_t-L_{t+N}\mid s_t].
\end{aligned}
\]
The entropy terms cancel, and $L_t$ is determined by $s_t$. Thus
\[
\ValuePi(s_t)=L_t-\E_{N,\pi}[L_{t+N}\mid s_t].
\]
The value in Equation~\eqref{eq:valueid} is therefore the expected reduction in KL error from continuing under $\pi$ to the geometric horizon.
\end{proof}

\section{Additional Experimental Results}
\label{sec:exp_supp}

This appendix reports the hyperparameter sweep, validation accuracy, the ablation, reliability diagrams, Brier scores, checkpoint selection sensitivity, and critic validation, in the order of Section~\ref{sec:res}. The canonical five-seed test results and confidence intervals are reported in Tables~\ref{tab:non-medmnist-acc-ece-by-method} and~\ref{tab:medmnist-acc-ece-by-method}.{}

\subsection{Hyperparameters}
\label{sec:hyperparameters}

A short overview of the hyperparameter sweep and selected configuration is shown in Table~\ref{tab:params}.
The sweep selected this shared optimizer configuration for every method and dataset. No learning rate scheduler was used.
{Training schedules vary across tasks and figure panels. The training dynamics in Figure~\ref{fig:dynamics} use $1000$ total training epochs for both SL and RIC-SPO.}
To make compute budgets comparable to our RIC method with discount $\gamma=0.8$, we matched both adaptive methods to the same geometric horizon. We set PonderNet's geometric prior to $\lambda_p = 1-\gamma = 0.2$ (implying $\E[N]=1/\lambda_p=5$) and used the standard KL weight $\beta=0.01$, while keeping $\texttt{max\_steps}=20$.
For ACT, we fixed the numerical/stability constants to paper defaults ($\epsilon=0.01$, positive halting-bias initialization) and set the ponder penalty to $\tau = -\ln(\gamma)=0.223$, which corresponds to a geometric continuation probability $\gamma$ under a linear compute cost.
For EDL, the evidential head used the standard evidence parametrization, converting pre-activations to evidence via clipped exponentiation and Dirichlet parameters as $\alpha
  = \exp(\mathrm{clip}(z, -10, 10)) + 1 $.
  Models were optimized with the Dirichlet Bayes risk cross-entropy loss plus a KL regularizer to the
  uniform Dirichlet prior with weight 1.0, linearly annealed over the
  first $10$ epochs.
These settings give an analytically matched starting point for compute comparisons, while each method still learns its own step count.

For RIC-SPO, we bound the Dirichlet concentration to $c\in[1,K]$ to preserve stochastic exploration during training.{} The SPO objective uses a clipping ratio of $0.2$ and the critic is trained with GAE with $\lambda=0.5$.

For low resolution datasets such as CIFAR-10, we modify the encoder stem to better preserve spatial information, using a $3 \times 3$ stride 1 convolution and removing the initial max pooling layer to avoid early spatial downsampling. The original ResNet architecture is otherwise preserved, with the activation function selected from the sweep in Table~\ref{tab:params}. Encoding and latent state dimensions are set equal.
All experiments were conducted on NVIDIA L40 GPUs.

\begin{table}[tb]
  \centering
  \caption{Hyperparameter search space and selected configuration.}
  \label{tab:params}
  \begin{tabular}{llll}
    \toprule
    Group & Hyperparameter & Search space & Selected \\
    \midrule
    Optimizer & Choice & \{SGD, AdamW\} & AdamW \\
    Optimizer & Learning rate & $\{3 \times 10^{-4}, 0.1\}$ & $3 \times 10^{-4}$ \\
    Optimizer & Weight decay & -- & $10^{-4}$ \\
    Optimizer & Gradient clipping & -- & $0.5$ \\
    Training & Batch size & \{128, 256\} & 256 \\
    Model & Activation & \{ReLU, SiLU\} & SiLU \\
    \bottomrule
  \end{tabular}
\end{table}

\subsection{Generalization}
\label{app:generalization}

Figure~\ref{fig:lgeneralization} shows validation accuracy as a function of training accuracy for CIFAR-10 (a), Galaxy10 (b), and ImageWoof (c). At matched training accuracy, deterministic RIC-SPO attains higher validation accuracy on all three datasets.

\begin{figure}[tb]
    \centering
    \includegraphics[width=\textwidth]{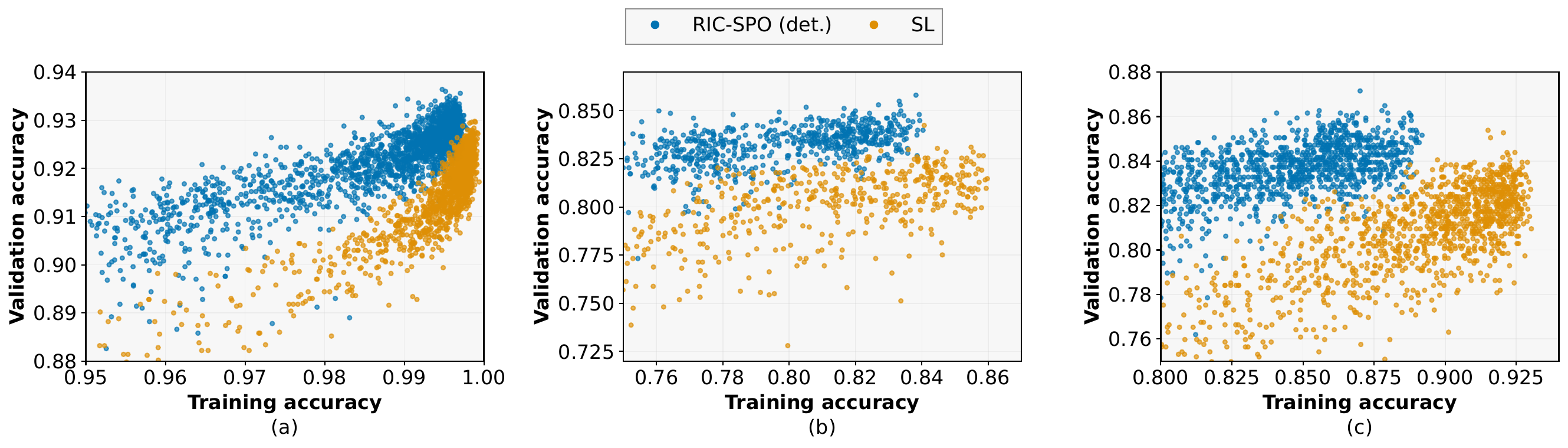}
    \caption{\textbf{Validation accuracy as a function of training accuracy.}
    (a) CIFAR-10, (b) Galaxy10, and (c) ImageWoof.
    At matched training accuracy, deterministic RIC-SPO attains higher validation accuracy on all three datasets.}
    \label{fig:lgeneralization}
\end{figure}

\subsection{Ablation Study}
\label{app:ablation}

Figure~\ref{fig:abl} {compares the policy head and optimization algorithm}. The Dirichlet head improves stability and training speed over a logistic Gaussian head.
Replacing SPO with PPO gives RIC-PPO, a second algorithm for the same MDP. In this ablation, its faster training comes with higher calibration error. SPO is designed to constrain probability ratios within the trust region more effectively than PPO \citep{xie2025simple}.

\begin{figure}[tb]
    \centering
    \includegraphics[width=\textwidth]{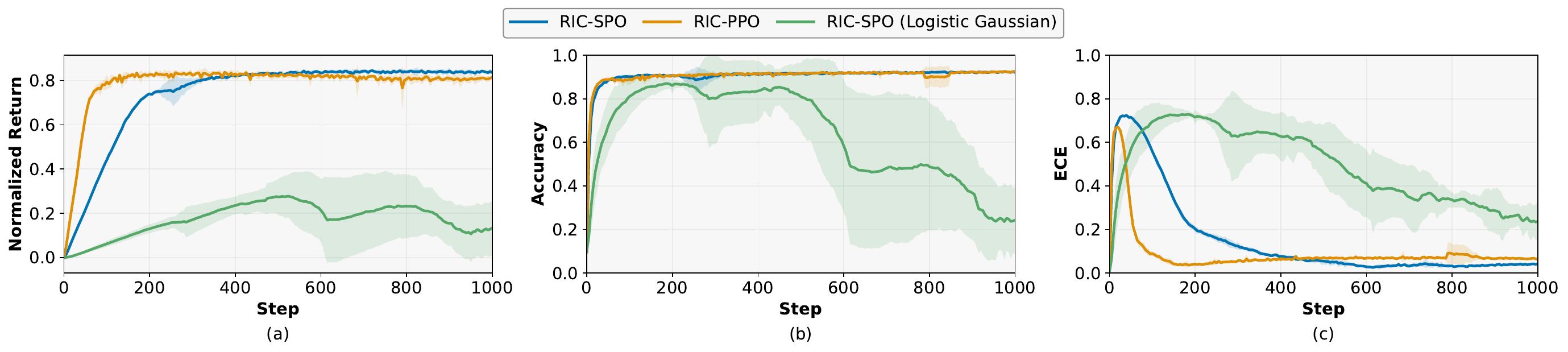}
    \caption{\textbf{Ablation study.}
    (a) Normalized return, (b) validation accuracy, and (c) ECE over training time for ablations.
    Standard RIC-SPO (blue) with Dirichlet head outperforms the logistic Gaussian head in stability and training speed. RIC-PPO, the same MDP trained with a PPO update, trains faster and {has higher ECE}.}
    \label{fig:abl}
\end{figure}

\subsection{\texorpdfstring{{Additional Calibration Results}}{Additional Calibration Results}}
\label{app:additional-calibration}

Figures~\ref{fig:ece_cifar} and~\ref{fig:ece_galaxy} show the corresponding confidence distributions and reliability diagrams on CIFAR-10 and Galaxy10.

\begin{figure}[tb]
    \centering
    \includegraphics[width=\textwidth]{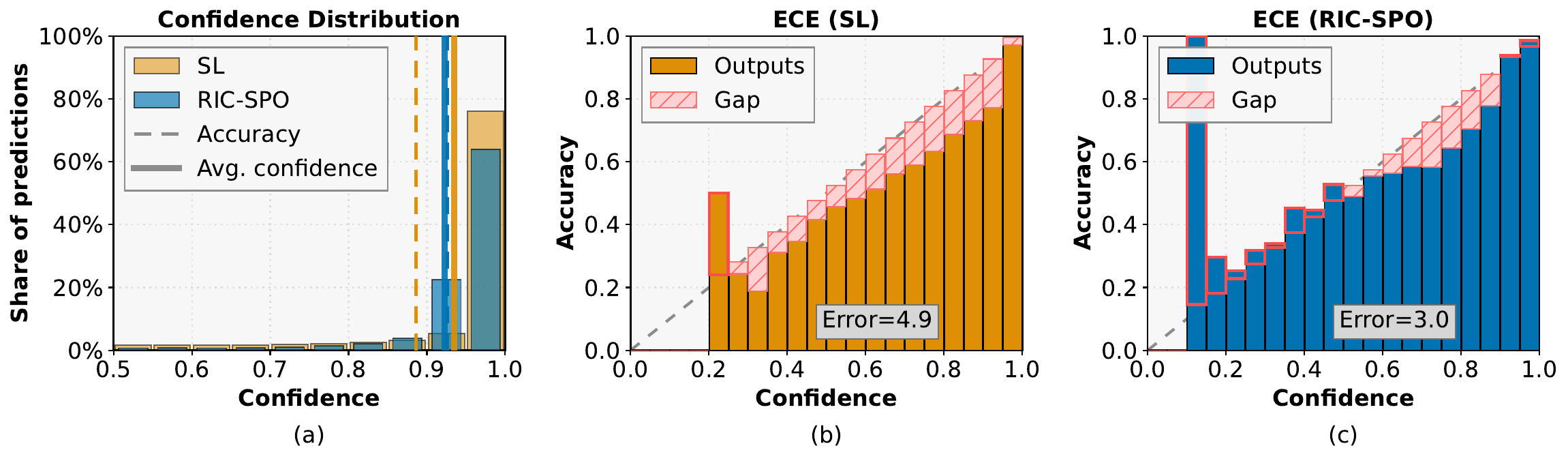}
    \caption{\textbf{Confidence distributions and reliability diagrams on CIFAR-10 (test).}
    The confidence histogram and reliability diagrams aggregate five models selected by validation loss; RIC-SPO is evaluated deterministically.}
    \label{fig:ece_cifar}
\end{figure}

\begin{figure}[tb]
    \centering
    \includegraphics[width=\textwidth]{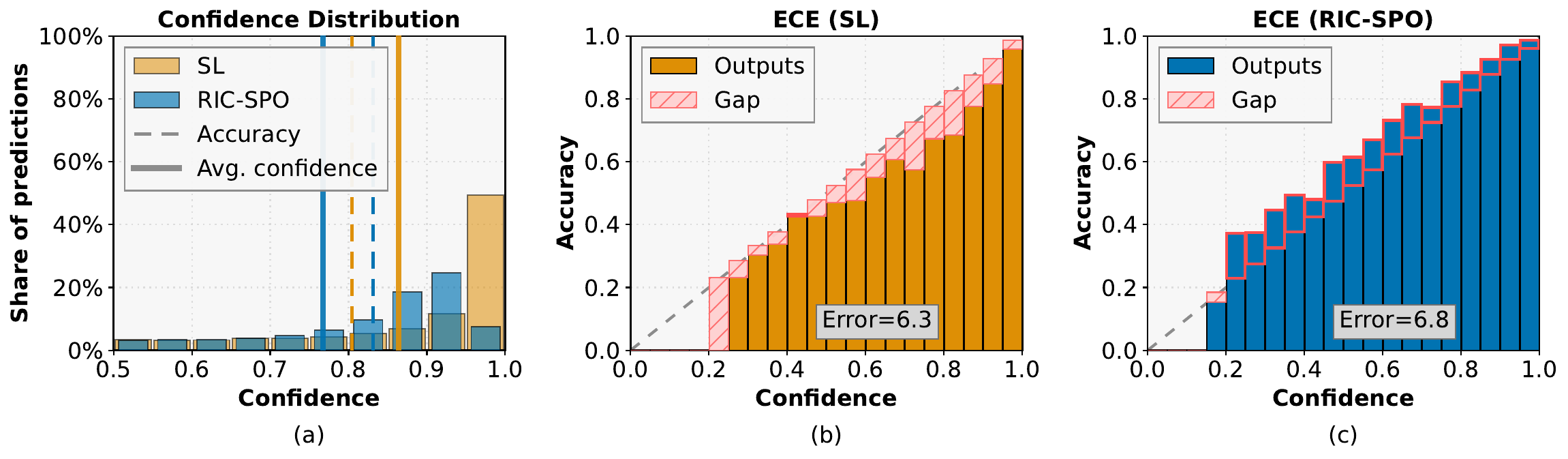}
    \caption{\textbf{Confidence distributions and reliability diagrams on Galaxy10 (test).}
    The confidence histogram and reliability diagrams aggregate five models selected by validation loss; RIC-SPO is evaluated deterministically.}
    \label{fig:ece_galaxy}
\end{figure}

\subsection{Brier Score}
\label{app:brier}

Table~\ref{tab:brier} reports the half-scaled multiclass Brier score
$\frac{1}{2N}\sum_{i=1}^{N}\sum_{k=1}^{K}(p_{ik}-y_{ik})^2$, for which lower is better.
We evaluate the checkpoints selected by the lowest validation loss, as in the main tables.
SL, EDL, and deterministic RIC-SPO use one evaluation pass; stochastic RIC-SPO averages ten passes.
Deterministic RIC-SPO has the lowest mean or an overlapping confidence interval on eight of the ten datasets.

\begin{table}[tb]
  \centering
  
  \scriptsize
  \setlength{\tabcolsep}{5pt}
  \renewcommand{\arraystretch}{1.08}
  \begin{tabular}{@{}lcccc@{}}
    \toprule
    Dataset & SL & EDL & RIC-SPO (det.) & RIC-SPO (stoch.) \\
    \midrule
    CIFAR-10
      & $0.086{\pm}0.003$ & $0.070{\pm}0.002$ & $\mathbf{0.060{\pm}0.001}$ & $0.084{\pm}0.002$ \\
    CIFAR-10N (Agg.)
      & $0.103{\pm}0.003$ & $0.128{\pm}0.006$ & $\mathbf{0.095{\pm}0.001}$ & $0.113{\pm}0.004$ \\
    CIFAR-10N (Worst)
      & $\mathbf{0.194{\pm}0.008}$ & $0.336{\pm}0.013$ & $0.243{\pm}0.006$ & $0.261{\pm}0.008$ \\
    ImageWoof
      & $0.136{\pm}0.005$ & $0.125{\pm}0.002$ & $\mathbf{0.109{\pm}0.003}$ & $0.127{\pm}0.006$ \\
    Galaxy10
      & $0.145{\pm}0.002$ & $0.156{\pm}0.004$ & $\mathbf{0.129{\pm}0.004}$ & $0.146{\pm}0.004$ \\
    \midrule
    OrganMNIST
      & $0.156{\pm}0.003$ & $\mathbf{0.149{\pm}0.002}$ & $\mathbf{0.137{\pm}0.010}$ & $\mathbf{0.143{\pm}0.002}$ \\
    DermaMNIST
      & $\mathbf{0.162{\pm}0.002}$ & $0.184{\pm}0.005$ & $\mathbf{0.159{\pm}0.004}$ & $0.176{\pm}0.003$ \\
    OCTMNIST
      & $\mathbf{0.112{\pm}0.023}$ & $\mathbf{0.125{\pm}0.016}$ & $\mathbf{0.124{\pm}0.012}$ & $\mathbf{0.122{\pm}0.008}$ \\
    TissueMNIST
      & $\mathbf{0.182{\pm}0.001}$ & $0.230{\pm}0.015$ & $0.206{\pm}0.002$ & $0.235{\pm}0.003$ \\
    PneumoniaMNIST
      & $\mathbf{0.097{\pm}0.020}$ & $\mathbf{0.089{\pm}0.009}$ & $\mathbf{0.093{\pm}0.016}$ & $\mathbf{0.095{\pm}0.017}$ \\
    \bottomrule
  \end{tabular}
  \caption{\textbf{Test Brier score under validation-loss checkpoint selection.}
  Entries are means and 95\% confidence intervals over five seeds.
  \textbf{Bold:} lowest mean and any method whose reported interval overlaps it.}
  \label{tab:brier}
  
\end{table}

\subsection{Checkpoint Selection Sensitivity}
\label{app:selector-sensitivity}

Tables~\ref{tab:non-medmnist-best-acc} and~\ref{tab:medmnist-best-acc} repeat the SL, EDL, and RIC-SPO evaluation with checkpoints selected by the highest validation accuracy. The accuracy conclusions are stable under the change of selector, and deterministic RIC-SPO keeps the strongest accuracy profile. {Calibration is more sensitive to checkpoint selection.} Validation NLL favors checkpoints where the full predictive distribution scores well, while validation accuracy ignores confidence and selects checkpoints with the same argmax and more extreme probabilities.

\begin{table}[tb]
  \centering
  
  \scriptsize
  \setlength{\tabcolsep}{3.0pt}
  \renewcommand{\arraystretch}{1.08}
  \begin{tabular}{@{}ll*{5}{c}@{}}
    \toprule
    Method & Metric & CIFAR-10 & CIFAR-10N (Agg.) & CIFAR-10N (Worst) & ImageWoof & Galaxy10 \\
    \midrule
    \textbf{SL} & Acc. [\%] & $92.71{\pm}0.29$ & $87.45{\pm}0.49$ & $\mathbf{77.29{\pm}1.08}$ & $83.80{\pm}0.66$ & $80.61{\pm}0.48$ \\
     & ECE & $0.058{\pm}0.002$ & $0.100{\pm}0.004$ & $\mathbf{0.168{\pm}0.018}$ & $0.098{\pm}0.009$ & $0.101{\pm}0.012$ \\
    \addlinespace[2pt]
    \textbf{EDL} & Acc. [\%] & $92.73{\pm}0.38$ & $87.45{\pm}0.40$ & $74.41{\pm}1.37$ & $83.58{\pm}0.17$ & $80.57{\pm}0.53$ \\
     & ECE & $0.052{\pm}0.002$ & $\mathbf{0.082{\pm}0.003}$ & $0.516{\pm}0.031$ & $0.074{\pm}0.007$ & $0.107{\pm}0.024$ \\
    \addlinespace[2pt]
    \textbf{RIC-det.} & Acc. [\%] & $\mathbf{93.37{\pm}0.16}$ & $\mathbf{88.67{\pm}0.60}$ & $\mathbf{77.93{\pm}0.48}$ & $\mathbf{85.32{\pm}0.62}$ & $\mathbf{82.85{\pm}0.91}$ \\
     & ECE & $\mathbf{0.029{\pm}0.008}$ & $\mathbf{0.159{\pm}0.221}$ & $0.589{\pm}0.030$ & $0.048{\pm}0.017$ & $\mathbf{0.108{\pm}0.081}$ \\
    \addlinespace[2pt]
    \textbf{RIC-stoch.} & Acc. [\%] & $92.82{\pm}0.29$ & $\mathbf{88.02{\pm}0.23}$ & $66.75{\pm}1.55$ & $\mathbf{83.28{\pm}1.48}$ & $\mathbf{81.12{\pm}0.84}$ \\
     & ECE & $0.055{\pm}0.002$ & $\mathbf{0.086{\pm}0.004}$ & $\mathbf{0.177{\pm}0.097}$ & $\mathbf{0.017{\pm}0.007}$ & $\mathbf{0.026{\pm}0.009}$ \\
    \bottomrule
  \end{tabular}
  \caption{\textbf{Test accuracy and calibration under validation-accuracy checkpoint selection.} Entries are mean and 95\% confidence interval over five seeds. \textbf{Bold:} best mean and any method whose reported interval overlaps it.}
  \label{tab:non-medmnist-best-acc}
  
\end{table}

\begin{table}[tb]
  \centering
  
  \scriptsize
  \setlength{\tabcolsep}{3.0pt}
  \renewcommand{\arraystretch}{1.08}
  \begin{tabular}{@{}ll*{5}{c}@{}}
    \toprule
    Method & Metric & OrganMNIST & DermaMNIST & OCTMNIST & TissueMNIST & PneumoniaMNIST \\
    \midrule
    \textbf{SL} & Acc. [\%] & $\mathbf{81.90{\pm}0.50}$ & $\mathbf{78.40{\pm}0.94}$ & $\mathbf{86.00{\pm}2.43}$ & $\mathbf{73.97{\pm}0.21}$ & $\mathbf{89.07{\pm}1.91}$ \\
     & ECE & $0.142{\pm}0.003$ & $0.168{\pm}0.010$ & $\mathbf{0.115{\pm}0.021}$ & $\mathbf{0.051{\pm}0.014}$ & $\mathbf{0.099{\pm}0.018}$ \\
    \addlinespace[2pt]
    \textbf{EDL} & Acc. [\%] & $\mathbf{81.43{\pm}0.70}$ & $76.29{\pm}0.81$ & $\mathbf{85.68{\pm}2.01}$ & $73.22{\pm}0.42$ & $\mathbf{90.51{\pm}2.32}$ \\
     & ECE & $0.108{\pm}0.014$ & $\mathbf{0.116{\pm}0.014}$ & $\mathbf{0.109{\pm}0.019}$ & $0.130{\pm}0.021$ & $\mathbf{0.084{\pm}0.021}$ \\
    \addlinespace[2pt]
    \textbf{RIC-det.} & Acc. [\%] & $\mathbf{81.03{\pm}1.32}$ & $\mathbf{78.29{\pm}1.00}$ & $\mathbf{86.72{\pm}2.53}$ & $73.16{\pm}0.19$ & $\mathbf{91.31{\pm}1.54}$ \\
     & ECE & $\mathbf{0.150{\pm}0.261}$ & $\mathbf{0.120{\pm}0.099}$ & $\mathbf{0.149{\pm}0.137}$ & $0.404{\pm}0.102$ & $\mathbf{0.076{\pm}0.018}$ \\
    \addlinespace[2pt]
    \textbf{RIC-stoch.} & Acc. [\%] & $\mathbf{80.82{\pm}1.16}$ & $\mathbf{77.52{\pm}1.04}$ & $\mathbf{85.67{\pm}2.45}$ & $70.10{\pm}0.34$ & $\mathbf{89.53{\pm}2.91}$ \\
     & ECE & $\mathbf{0.065{\pm}0.028}$ & $\mathbf{0.115{\pm}0.010}$ & $\mathbf{0.096{\pm}0.030}$ & $\mathbf{0.079{\pm}0.014}$ & $\mathbf{0.089{\pm}0.013}$ \\
    \bottomrule
  \end{tabular}
  \caption{\textbf{Test accuracy and calibration on MedMNIST under validation-accuracy checkpoint selection.} Entries are mean $\pm$ 95\% confidence interval across five seeds. \textbf{Bold} marks the best mean and any method whose reported interval overlaps that of the best mean.}
  \label{tab:medmnist-best-acc}
  
\end{table}

\subsection{Critic Validation Against Vine Returns}
\label{app:critic-validation}

We compare $\Critic(s_t)$ with vine estimates of the future return $G_t$ along stochastic and deterministic mean-policy trajectories. Stochastic targets average $128$ continuations per state; deterministic targets use the realized continuation. Figure~\ref{fig:critic-vine} reports five seeds on CIFAR-10 and both CIFAR-10N variants.

On CIFAR-10 and CIFAR-10N (Agg.), the critic makes the correct stop or continue decision on $69$ to $87\%$ of qualifying states under both deployments. Figure~\ref{fig:critic-vine} separates the two decisions and shows that the deterministic critic errs toward stopping early. At $40\%$ label noise, the deterministic critic favors early stopping: it detects $83\%$ of harmful continuations and $44\%$ of beneficial continuations.

\begin{figure}[tb]
    \centering
    \includegraphics[width=\textwidth]{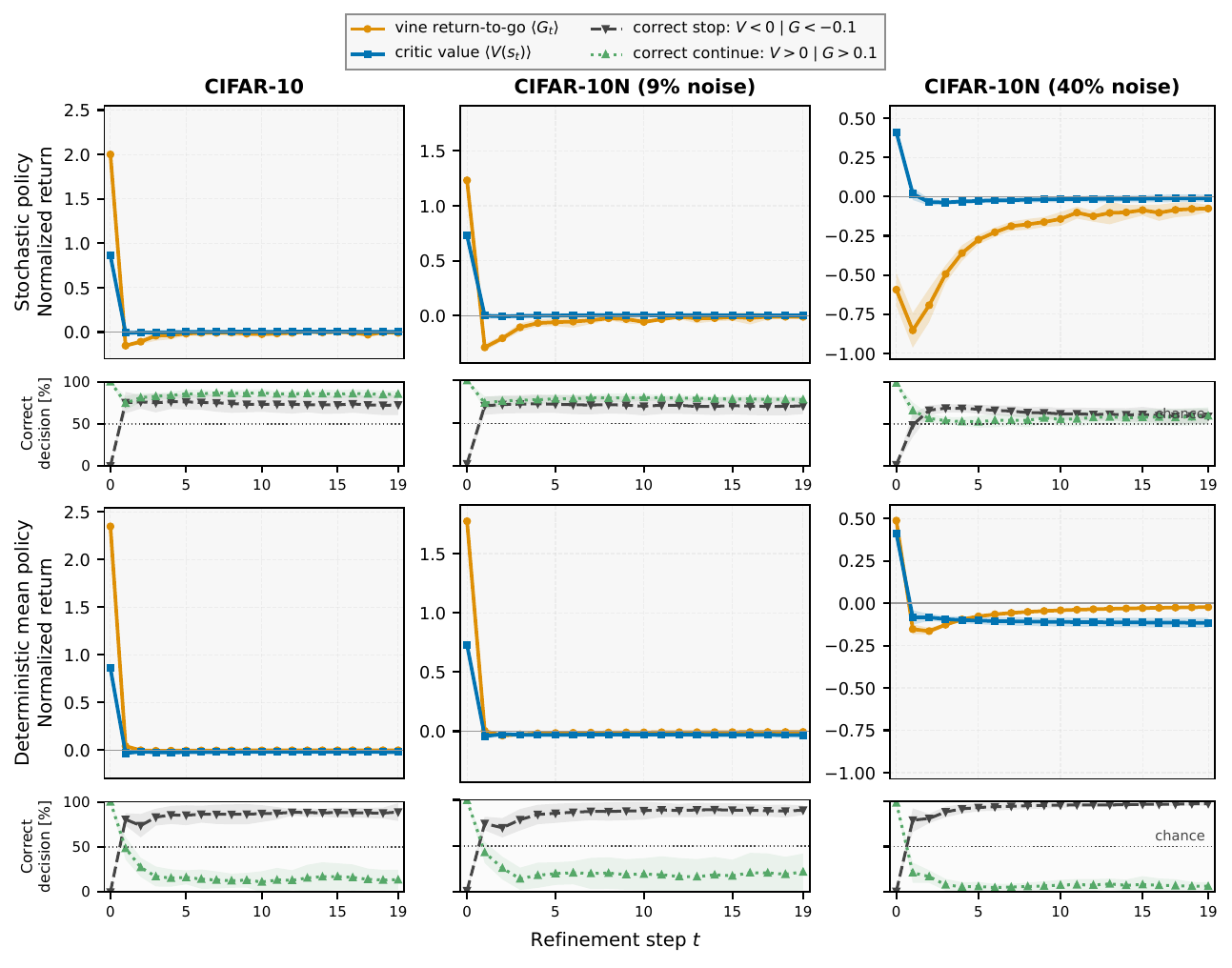}
    \caption{\textbf{Critic validation against vine returns.}
    Mean return to go $G_t$ and critic value {$\Critic(s_t)$}, with correct stop and continue decisions below, for stochastic (top) and deterministic mean-policy trajectories (bottom). Bands show $95\%$ confidence intervals over five seeds.}
    \label{fig:critic-vine}
\end{figure}

\end{document}